\documentclass{article}

\usepackage{graphicx} 

\usepackage{natbib}
\usepackage{times}
\usepackage{algorithm}
\usepackage{algorithmic}
\usepackage[noend]{algpseudocode}
\usepackage{subcaption}
\usepackage{amsmath}
\usepackage{amsfonts}
\usepackage{latexsym}
\usepackage{xspace}
\usepackage{xcolor}
\usepackage{dsfont}
\usepackage{booktabs}
\usepackage{siunitx} 
\usepackage{hyperref}

\newtheorem{theorem}{Theorem}[section]
\newtheorem{lemma}[theorem]{Lemma}
\newtheorem{proposition}[theorem]{Proposition}

\newtheorem{definition}[theorem]{Definition}

\newenvironment{proof}{\par\noindent\textit{Proof.}}{$\Box$\par\bigskip\par}

\newcommand{\opt}{{\operatorname{OPT}}}
\newcommand{\bbRp}{{\mathbb{R}_{\ge 0}}}

\newcommand{\cA}{\mathcal{A}}
\newcommand{\RALG}{\textsc{PRo}\xspace}
\newcommand{\GRALG}{\textsc{PRo-Gr}\xspace}
\newcommand{\SRALG}{\textsc{PRo-St}\xspace}
\newcommand{\SAT}{\textsc{Saturate}\xspace}
\newcommand{\THRESH}{\textsc{Thresholding-Greedy}\xspace}
\newcommand{\STOCHG}{\textsc{Stochastic-Greedy}\xspace}
\newcommand{\LGREEDY}{\textsc{Lazy-Greedy}\xspace}
\newcommand{\OSU}{\textsc{OSU}\xspace}
\newcommand{\GREEDY}{\textsc{Greedy}\xspace}

\newcommand{\ceillogtau}{{\lceil \log{\tau} \rceil}}
\newcommand\given[1][]{\:#1\vert\:}

\newcommand{\alphan}{\alpha_{\mathrm{next}}}

\DeclareMathOperator*{\argmin}{argmin}
\DeclareMathOperator*{\argmax}{argmax}
\usepackage[accepted]{icml2017} 

\icmltitlerunning{Robust Submodular Maximization: A Non-Uniform Partitioning Approach}

\begin{document} 

\twocolumn[
\icmltitle{Robust Submodular Maximization: \\ A Non-Uniform Partitioning Approach}



\icmlsetsymbol{equal}{*}

\begin{icmlauthorlist}
\icmlauthor{Ilija Bogunovic}{to}
\icmlauthor{Slobodan Mitrovi\'c}{other}
\icmlauthor{Jonathan Scarlett}{to}
\icmlauthor{Volkan Cevher}{to}
\end{icmlauthorlist}


\icmlaffiliation{to}{LIONS, EPFL, Switzerland}
\icmlaffiliation{other}{LTHC, EPFL, Switzerland}

\icmlcorrespondingauthor{Ilija Bogunovic}{ilija.bogunovic@epfl.ch}
\icmlcorrespondingauthor{Slobodan Mitrovi\'c}{slobodan.mitrovic@epfl.ch}
\icmlcorrespondingauthor{Jonathan Scarlett}{jonathan.scarlett@epfl.ch}
\icmlcorrespondingauthor{Volkan Cevher}{volkan.cevher@epfl.ch}

%
%

\icmlkeywords{Submodular maximization, robust optimization, influence maximization, personalization}

\vskip 0.3in
]

\printAffiliationsAndNotice

\begin{abstract} 
We study the problem of maximizing a monotone submodular function subject to a cardinality constraint $k$, with the added twist that a number of items $\tau$ from the returned set may be removed. We focus on the worst-case setting considered in \cite{orlin2016robust}, in which a constant-factor approximation guarantee was given for $\tau = o(\sqrt{k})$.  In this paper, we solve a key open problem raised therein, presenting a new Partitioned Robust ($\RALG$) submodular maximization algorithm  that achieves the same guarantee for more general $\tau = o(k)$.  Our algorithm constructs partitions consisting of buckets with exponentially increasing sizes, and applies standard submodular optimization subroutines on the buckets in order to construct the robust solution. We numerically demonstrate the performance of $\RALG$ in data summarization and influence maximization, demonstrating gains over both the greedy algorithm and the algorithm of \cite{orlin2016robust}.
\end{abstract} 
\vspace*{1ex}
\section{Introduction}

Discrete optimization problems arise frequently in machine learning, and are often NP-hard even to approximate. In the case of a set function exhibiting {\em submodularity}, one can efficiently perform maximization subject to cardinality constraints with a $\big(1 - \frac{1}{e}\big)$-factor approximation guarantee.  Applications include influence maximization~\cite{kempe2003maximizing}, document summarization~\cite{lin2011class}, sensor placement~\cite{krause2007near}, and active learning~\cite{krause2012submodular}, just to name a few.

In many applications of interest, one requires {\em robustness} in the solution set returned by the algorithm, in the sense that the objective value degrades as little as possible when some elements of the set are removed.  For instance, (i) in influence maximization problems, a subset of the chosen users may decide not to spread the word about a product; (ii) in summarization problems, a user may choose to remove some items from the summary due to their personal preferences; (iii) in the problem of sensor placement for outbreak detection, some of the sensors might fail. 

In situations where one does not have a reasonable prior distribution on the elements removed, or where one requires robustness guarantees with a high level of certainty, protecting against worst-case removals becomes important. This setting results in the {\em robust submodular function maximization} problem, in which we seek to return a set of cardinality $k$ that is robust with respect to the worst-case removal of $\tau$ elements. 

The robust problem formulation was first introduced in \cite{krause2008robust}, and was further studied in \cite{orlin2016robust}.  In fact, \cite{krause2008robust} considers a more general formulation where a constant-factor approximation guarantee is impossible in general, but shows that one can match the optimal (robust) objective value for a given set size at the cost of returning a set whose size is larger by a logarithmic factor.  In contrast, \cite{orlin2016robust} designs an algorithm that obtains the first constant-factor approximation guarantee to the above problem when $\tau = o(\sqrt{k})$.  A key difference between the two frameworks is that the algorithm complexity is exponential in $\tau$ in \cite{krause2008robust}, whereas the algorithm of \cite{orlin2016robust} runs in polynomial time.

\begin{table*}[htb]
 	\vspace*{2ex}
    \begin{tabular}{ccccc} \toprule
        {Algorithm} & {Max. Robustness} & {Cardinality} & {Oracle Evals.} & {Approx.} \\ \midrule
        \SAT \textsc{\cite{krause2008robust}} & Arbitrary & $k(1 + \Theta(\log(\tau k \log n)))$ & exponential in $\tau$ & 1.0  \\
        \OSU \textsc{\cite{orlin2016robust}}  & $o(\sqrt{k})$  & $k$ & $\mathcal{O}(nk)$  & 0.387   \\  \midrule
        \textsc{\RALG-\GREEDY (Ours)}  & $o(k)$   & $k$ & $\mathcal{O}(nk)$   & 0.387  \\  \bottomrule
    \end{tabular}
    \caption{Algorithms for robust monotone submodular optimization with a cardinality constraint. The proposed algorithm is efficient and allows for greater robustness.}\label{results}
    \vspace*{2ex}
\end{table*}

\textbf{Contributions.} In this paper, we solve a key open problem posed in \cite{orlin2016robust}, namely, whether a constant-factor approximation guarantee is possible for general $\tau = o(k)$, as opposed to only $\tau = o(\sqrt{k})$.  We answer this question in the affirmative, providing a new Partitioned Robust ($\RALG$) submodular maximization algorithm that attains a constant-factor approximation guarantee; see Table~\ref{results} for comparison of different algorithms for robust monotone submodular optimization with a cardinality constraint.

Achieving this result requires novelty both in the algorithm and its mathematical analysis:  While our algorithm bears some similarity to that of \cite{orlin2016robust}, it uses a novel structure in which the constructed set is arranged into partitions consisting of buckets whose sizes increase exponentially with the partition index.  A key step in our analysis provides a recursive relationship between the objective values attained by buckets appearing in adjacent partitions. 

In addition to the above contributions, we provide the first empirical study beyond what is demonstrated for $\tau = 1$ in \cite{krause2008robust}.  We demonstrate several scenarios in which our algorithm outperforms both the greedy algorithm and the algorithm of \cite{orlin2016robust}.

%

\section{Problem Statement}
Let $V$ be a ground set with cardinality $|V|=n$, and let $f:2^V \to \bbRp$ be a set function defined on $V$. The function $f$ is said to be \emph{submodular} if for any sets $X \subseteq Y \subseteq V$ and any element $e \in V \setminus Y$, it holds that
\[
    f(X \cup \lbrace e \rbrace) - f(X) \ge f(Y \cup \lbrace e \rbrace) - f(Y).
\]
We use the following notation to denote the marginal gain in the function value due to adding the elements of a set $Y$ to the set $X$:
\[f(Y|X):= f(X \cup Y) - f(X). \] 
In the case that $Y$ is a singleton of the form $\{e\}$, we adopt the shorthand $f(e|X)$. We say that $f$ is \emph{monotone} if for any sets $X \subseteq Y \subseteq V$ we have $f(X) \leq f(Y)$, and \emph{normalized} if $f(\emptyset) = 0$.

The problem of maximizing a normalized monotone submodular function subject to a cardinality constraint, i.e.,
\begin{equation}
    \label{eq:card_prblm}
    \max_{S \subseteq V, |S| \leq k} f(S),
\end{equation}
has been studied extensively. A celebrated result of~\cite{nemhauser1978analysis} shows that a simple greedy algorithm that starts with an empty set and then iteratively adds elements with highest marginal gain provides a $(1 - 1/e)$-approximation.

In this paper, we consider the following \emph{robust} version of \eqref{eq:card_prblm}, introduced in~\cite{krause2008robust}:
\begin{equation}
    \label{eq:robust-card-prblm}
    \max_{S \subseteq V, |S| \leq k} \  \min_{Z \subseteq S, |Z|\leq{\tau}}\  f(S\setminus Z)
\end{equation}
We refer to $\tau$ as the robustness parameter, representing the size of the subset $Z$ that is removed from the selected set $S$. Our goal is to find a set $S$ such that it is robust upon the worst possible removal of $\tau$ elements, i.e., after the removal, the objective value should remain as large as possible. For $\tau=0$, our problem reduces to Problem~\eqref{eq:card_prblm}.

The greedy algorithm, which is near-optimal for Problem~\eqref{eq:card_prblm} can perform arbitrarily badly for Problem~\eqref{eq:robust-card-prblm}.  As an elementary example, let us fix $\epsilon \in [0, n-1)$ and $n \geq 0$, and consider the non-negative monotone submodular function given in Table~\ref{eq:counterexample}. For $k=2$, the greedy algorithm selects $\lbrace s_1, s_2 \rbrace$. The set that maximizes $\min_{s \in S}f(S \setminus {s})$ (i.e., $\tau=1$) is $\lbrace {s_1, s_3} \rbrace$. For this set, $\min_{s \in \lbrace s_1, s_2 \rbrace}f(\lbrace s_1,s_2 \rbrace \setminus {s}) = n-1$, while for the greedy set the robust objective value is $\epsilon$. As a result, the greedy algorithm can perform arbitrarily worse. 

In our experiments on real-world data sets (see Section~\ref{sec:experiments}), we further explore the empirical behavior of the greedy solution in the robust setting. Among other things, we observe that the greedy solution tends to be less robust when  the objective value largely depends on the first few elements selected by the greedy rule.

\begin{table}
    \vspace*{2ex}
    \centering
    \begin{tabular}{ccc} \toprule
        {$S$} & {$f(S)$} & {$\min_{s \in S}f(S \setminus {s})$}  \\ \midrule
        $\emptyset$ & $0$ & $0$  \\
        $\lbrace s_1 \rbrace$  & $n$   & $0$  \\
        $\lbrace s_2 \rbrace$  & $\epsilon$   & $0$ \\  
        $\lbrace s_3 \rbrace$  & $n-1$ & $0$ \\
        $\lbrace s_1, s_2 \rbrace$  & $n + \epsilon$ & $\epsilon$ \\
        $\lbrace s_1, s_3 \rbrace$  & $n$ & $\mathbf{n-1}$ \\
        $\lbrace s_2, s_3 \rbrace$  & $n$ & $\epsilon$ \\
        \bottomrule
    \end{tabular}
    \caption{Function $f$ used to demonstrate that $\GREEDY$ can perform arbitrarily badly.}
    \vspace*{2ex}
    \label{eq:counterexample}
\end{table}

\textbf{Related work.} \cite{krause2008robust} introduces the following generalization of \eqref{eq:robust-card-prblm}:
\begin{equation}\label{eq: robustandreas}
    \max_{S \subseteq V, |S| \leq k}\; \min_{i \in \lbrace1,\cdots, n\rbrace} f_i(S),
\end{equation}
where $f_i$ are normalized monotone submodular functions. The authors show that this problem is inapproximable in general, but propose an algorithm $\SAT$ which, when applied to \eqref{eq:robust-card-prblm}, returns a set of size $ k(1 + \Theta (\log(\tau k \log n)))$ whose robust objective is at least as good as the optimal size-$k$ set. \SAT requires a number of function evaluations that is exponential in $\tau$, making it very expensive to run even for small values.  The work of \cite{powersconstrained} considers the same problem for different types of submodular constraints.

Recently, robust versions of submodular maximization have been applied to influence maximization. In \cite{he2016robust}, the formulation \eqref{eq: robustandreas} is used to optimize a worst-case approximation ratio. The confidence interval setting is considered in \cite{chen2016robust}, where two runs of the \GREEDY algorithm (one pessimistic and one optimistic) are used to optimize the same ratio. By leveraging connections to continuous submodular optimization, \cite{staibrobust} studies a related continuous robust budget allocation problem. 

\cite{orlin2016robust} considers the formulation in \eqref{eq:robust-card-prblm}, and provides the first constant $0.387$-factor approximation result, valid for $\tau = o( \sqrt{k} )$. The algorithm proposed therein, which we refer to via the authors’ surnames as \OSU, uses the greedy algorithm (henceforth referred to as \GREEDY) as a sub-routine $\tau + 1$ times.  On each iteration, \GREEDY is applied on the elements that are not yet selected on previous iterations, with these previously-selected elements ignored in the objective function.  In the first $\tau$ runs, each solution is of size $\tau \log k$, while in the last run, the solution is of size $k - \tau^2 \log k$. The union of all the obtained disjoint solutions leads to the final solution set. 


\section{Applications} \label{sec:apps}
In this section, we provide several examples of applications where the robustness of the solution is favorable. The objective functions in these applications are non-negative, monotone and submodular, and are used in our numerical experiments in Section \ref{sec:experiments}. 

\textbf{Robust influence maximization.}
The goal in the influence maximization problem is to find a set of $k$ nodes (i.e., a targeted set) in a network that maximizes some measure of influence. For example, this problem appears in viral marketing, where companies wish to spread the word of a new product by targeting the most influential individuals in a social network.  Due to poor incentives or dissatisfaction with the product, for instance, some of the users from the targeted set might make the decision not to spread the word about the product.

For many of the existing diffusion models used in the literature (e.g., see \cite{kempe2003maximizing}), given the targeted set $S$, the expected number of influenced nodes at the end of the diffusion process is a monotone and submodular function of $S$~\cite{he2016robust}.  For simplicity, we consider a basic model in which all of the neighbors of the users in $S$ become influenced, as well as those in $S$ itself. 

More formally, we are given a graph $G = (V, E)$, where $V$ stands for nodes and $E$ are the edges. For a set $S$, let $\mathcal{N}(S)$ denote all of its neighboring nodes. The goal is to solve the robust \emph{dominating set problem}, i.e., to find a set of nodes $S$ of size $k$ that maximizes 
\begin{equation}
    \min_{|R_S| \leq \tau, R_S \subseteq S} |(S \setminus R_S) \cup \mathcal{N}(S \setminus R_S)|, \label{eq:domset}
\end{equation}
where $R_S \subseteq S$ represents the users that decide not to spread the word. 
The non-robust version of this objective function has previously been considered in several different works, such as \cite{mirzasoleiman2015distributed} and~\cite{norouzi2016efficient}. 


\textbf{Robust personalized image summarization.}
In the personalized image summarization problem, a user has a collection of images, and the goal is to find $k$ images that are representative of the collection.   

After being presented with a solution, the user might decide to remove a certain number of images from the representative set due to various reasons (e.g., bad lighting, motion blur, etc.). Hence, our goal is to find a set of images that remain good representatives of the collection even after the removal of some number of them.

One popular way of finding a representative set in a massive dataset is via exemplar based clustering, i.e., by minimizing the sum of pairwise dissimilarities between the exemplars $S$ and the elements of the data set $V$. This problem can be posed as a submodular maximization problem subject to a cardinality constraint; \textit{cf.}, \cite{lucic16horizontally}. 

Here, we are interested in solving the robust summarization problem, i.e., we want to find a set of images $S$ of size $k$ that maximizes
\begin{align}
    \min_{|R_S| \leq \tau, R_S \subseteq S} f(\lbrace e_0 \rbrace) - f((S \setminus R_S) \cup \lbrace e_0 \rbrace), \label{eq:exemplar}
\end{align}
where $e_0$ is a reference element and $f(S) = \frac{1}{|V|} \sum_{v \in V} \min_{s \in S} d(s,v)$ is the $k$-\emph{medoid} loss function, and where $d(s,v)$ measures the dissimilarity between images $s$ and $v$.

Further potential applications not covered here include robust sensor placement~\cite{krause2008robust}, robust protection of networks~\cite{bogunovic2012robust}, and robust feature selection~\cite{globerson2006nightmare}.

\section{Algorithm and its Guarantees}

\subsection{The algorithm}

Our algorithm, which we call the Partitioned Robust ($\RALG$) submodular maximization algorithm, is presented in Algorithm~\ref{algorithm:ralg}. As the input, we require a non-negative monotone submodular function $f:2^V \to \bbRp$, the ground set of elements $V$, and an optimization subroutine $\mathcal{A}$. The subroutine $\mathcal{A}(k',V')$ takes a cardinality constraint $k'$ and a ground set of elements $V'$. Below, we describe the properties of $\mathcal{A}$ that are used to obtain approximation guarantees.

The output of the algorithm is a set $S \subseteq V$ of size $k$ that is robust against the worst-case removal of $\tau$ elements. The returned set consists of two sets $S_0$ and $S_1$, illustrated in Figure~\ref{fig-S}.  $S_1$ is obtained by running the subroutine $\cA$ on $V \setminus S_0$ (i.e., ignoring the elements already placed into $S_0$), and is of size $k - |S_0|$. 


We refer to the set $S_0$ as the robust part of the solution set $S$. It consists of $\ceillogtau + 1$ partitions, where every partition $i \in \lbrace 0, \cdots ,\ceillogtau \rbrace$ consists of $\lceil \tau / 2^i \rceil$ buckets $B_j$, $j \in \lbrace 1, \cdots, \lceil \tau / 2^i \rceil \rbrace$. In partition $i$, every bucket contains $2^i \eta$ elements, where $\eta \in \mathbb{N}_+$ is a parameter that is arbitrary for now; we use $\eta = \log^2 k$ in our asymptotic theory, but our numerical studies indicate that even $\eta =1$ works well in practice. Each bucket $B_j$ is created afresh by using the subroutine $\cA$ on $V \setminus S_{0,\mathrm{prev}}$, where $S_{0,\mathrm{prev}}$ contains all elements belonging to the previous buckets.

The following proposition bounds the cardinality of $S_0$, and is proved in the supplementary material.

\begin{proposition}
\label{proposition:S_0_size}
Fix $k \geq \tau$ and $\eta \in \mathbb{N}_+$.  The size of the robust part $S_0$ constructed in Algorithm~\ref{algorithm:ralg} is
\begin{equation}
    |S_0|\ = \ \sum_{i=0}^{\ceillogtau} \lceil \tau / 2^i \rceil 2^i \eta  \ \le \ 3\eta \tau (\log k + 2) \nonumber.
\end{equation}
\end{proposition}

This proposition reveals that the feasible values of $\tau$ (i.e., those with $|S_0| \le k$) can be as high as $O\big( \frac{k}{\eta \tau}\big)$.  We will later set $\eta = O(\log^2 k)$, thus permitting all $\tau = o(k)$ up to a few logarithmic factors.  In contrast, we recall that the algorithm \OSU proposed in \cite{orlin2016robust} adopts a simpler approach where a robust set is used consisting of $\tau$ buckets of equal size $\tau \log k$, thereby only permitting the scaling $\tau = o(\sqrt{k})$.

\begin{algorithm}
\vspace{1ex}
    \caption{Partitioned Robust Submodular optimization algorithm ($\RALG$) \label{algorithm:ralg}}
    \begin{algorithmic}[1]
        \Require  Set $V$, $k$, $\tau$, $\eta \in \mathbb{N}_+$, $\text{algorithm } \mathcal{A}$
        \Ensure Set $S \subseteq V$ such that $|S| \leq k$
            \State $S_0, S_1 \gets \emptyset$
            \For {$ i \gets 0 \textbf{ to } \lceil \log \tau \rceil$}
                \For {$j \gets 1 \textbf{ to } \lceil \tau / 2^i \rceil$}
                    \State $B_j \gets \mathcal{A}\ (2^i \eta,\ (V\setminus S_0))$
                    \State $S_0 \gets S_0 \cup B_j$
                \EndFor
            \EndFor 
        \State $S_1 \gets \mathcal{A}\ (k - |S_0|,\ (V \setminus S_0))$
    \State {$S \gets S_0 \cup S_1$}\\
    \Return $S$ 
    \end{algorithmic}
\vspace{1ex}
\end{algorithm}
\vspace{1ex}

We provide the following intuition as to why \RALG succeeds despite having a smaller size for $S_0$ compared to the algorithm given in \cite{orlin2016robust}. First, by the design of the partitions, there always exists a bucket in partition $i$ that at most $2^i$ items are removed from.  The bulk of our analysis is devoted to showing that the union of these buckets yields a sufficiently high objective value.  While the earlier buckets have a smaller size, they also have a higher objective value per item due to diminishing returns, and our analysis quantifies and balances this trade-off.  Similarly, our analysis quantifies the trade-off between how much the adversary can remove from the (typically large) set $S_1$ and the robust part $S_0$.

\subsection{Subroutine and assumptions}

\RALG accepts a subroutine $\cA$ as the input. We consider a class of algorithms that satisfy the \emph{$\beta$-iterative property}, defined below.  We assume that the algorithm outputs the final set in some specific order $(v_1,\dotsc,v_k)$, and we refer to $v_i$ as the \emph{$i$-th output element}.

\begin{definition} \label{def:beta}
Consider a normalized monotone submodular set function $f$ on a ground set $V$, and an algorithm $\cA$. Given any set $T \subseteq V$ and size $k$, suppose that $\cA$ outputs an ordered set $(v_1,\dotsc,v_k)$ when applied to $T$, and define $\cA_i(T) = \{v_1,\dotsc,v_i\}$ for $i \le k$. We say that $\cA$ satisfies the \emph{$\beta$-iterative property} if
\begin{equation} \label{eq:beta-property}
    f(\cA_{i+1}(T)) - f(\cA_{i}(T)) \geq \frac{1}{\beta} \max_{v \in T} f(v|\cA_i(T)). 
\end{equation}
\end{definition}

Intuitively, \eqref{eq:beta-property} states that in every iteration, the algorithm adds an element whose marginal gain is at least a $1/\beta$ fraction of the maximum marginal. This necessarily requires that $\beta \ge 1$.

{\bf Examples.} Besides the classic greedy algorithm, which satisfies~\eqref{eq:beta-property} with  $\beta = 1$, a good candidate for our subroutine is $\THRESH$~\cite{badanidiyuru2014}, which satisfies the $\beta$-iterative property with $\beta = 1 / (1 - \epsilon)$. This decreases the number of function evaluations to $\mathcal{O}(n / \epsilon \log n / \epsilon)$.

$\STOCHG$~\cite{mirzasoleiman2015} is another potential subroutine candidate. While it is unclear whether this algorithm satisfies the $\beta$-iterative property, it requires an even smaller number of function evaluations, namely, $\mathcal{O}(n  \log 1 / \epsilon)$.  We will see in Section \ref{sec:experiments} that $\RALG$ performs well empirically when used with this subroutine.  We henceforth refer to $\RALG$ used along with its appropriate subroutine as $\RALG\text{-}\GREEDY$, $\RALG\text{-}\THRESH$, and so on.

{\bf Properties.} The following lemma generalizes a classical property of the greedy algorithm \cite{nemhauser1978analysis,krause2012submodular} to the class of algorithms satisfying the $\beta$-iterative property.  Here and throughout the paper, we use $\opt(k, V)$ to denote the following optimal set for \emph{non-robust} maximization:
\[
    \opt(k, V) \in \argmax_{S \subseteq V, |S| = k}\;  f(S),
\]

\begin{lemma}
\label{lemma:l-k-beta}
Consider a normalized monotone submodular function $f:2^V \to \bbRp$ and an algorithm $\cA(T)$, $T \subseteq V$, that satisfies the $\beta$-iterative property in \eqref{eq:beta-property}. Let $\cA_l(T)$ denote the set returned by the algorithm $\cA(T)$ after $l$ iterations. Then for all $k,l \in \mathbb{N}_+$
\begin{equation}
    f(\cA_{l}(T)) \geq \left( 1 - e^{-\frac{l}{\beta k}} \right) f(\opt(k,T)).
\end{equation}
\end{lemma}

We will also make use of the following property, which is implied by the $\beta$-iterative property.

\begin{proposition}
\label{proposition:beta-nice-main-rule} 
Consider a submodular set function $f:2^V \to \bbRp$ and an algorithm $\cA$ that satisfies the $\beta$-iterative property for some $\beta \ge 1$. Then, for any $T \subseteq V$ and element $e \in V \setminus \cA(T)$, we have
\begin{align}
    f(e|\cA(T)) & \leq \beta \frac{f(\cA(T))}{k}. \label{eq:beta-nice-2}
\end{align}
\end{proposition}
Intuitively, \eqref{eq:beta-nice-2} states that the marginal gain of any non-selected element cannot be more than $\beta$ times the average objective value of the selected elements. 
This is one of the rules used to define the $\beta$-nice class of algorithms in \cite{mirrokni2015}; however, we note that in general, neither the $\beta$-nice nor $\beta$-iterative classes are a subset of one another.


\begin{figure}
    \centering
    \includegraphics[scale=0.3]{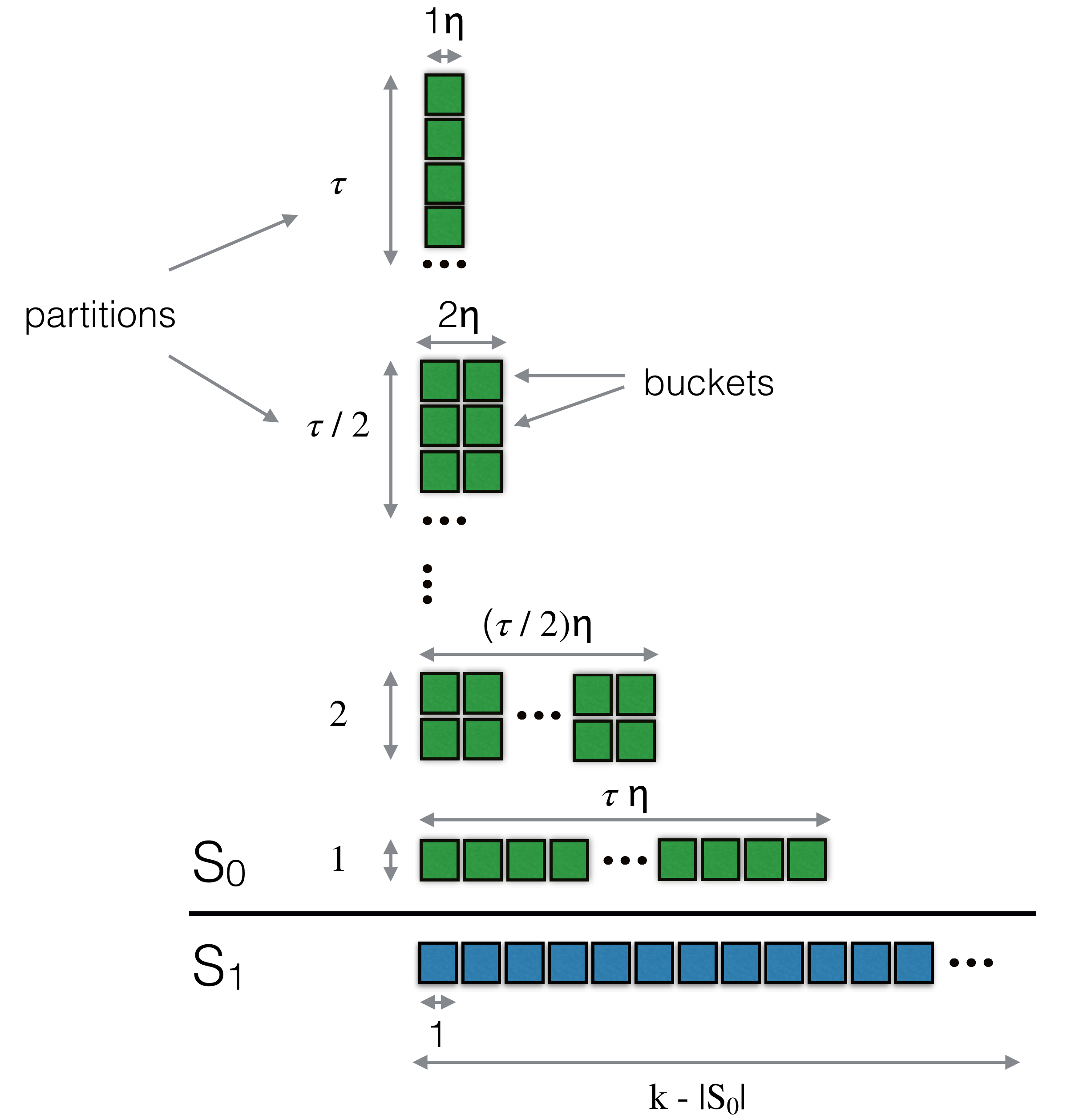}
    \caption{Illustration of the set $S = S_0 \cup S_1$ returned by $\RALG$. The size of $|S_1|$ is $k - |S_0|$, and the size of $|S_0|$ is given in Proposition~\ref{proposition:S_0_size}. Every partition in $S_0$ contains the same number of elements (up to rounding). 
    \label{fig-S}} \vspace*{2ex}
\end{figure}

\subsection{Main result: Approximation guarantee}

For the robust maximization problem, we let $\opt(k, V, \tau)$ denote the optimal set:
\[
    \opt(k, V, \tau) \in \argmax_{S \subseteq V, |S| = k}\; \min_{E \subseteq S, |E| \le \tau} f(S \setminus E).
\]
Moreover, for a set $S$, we let $E^*_{S}$ denote the minimizer
\begin{equation}
\label{eq:E}
 E^*_{S} \in \argmin_{E \subseteq S, |E| \le \tau}f(S \setminus E).
\end{equation}
With these definitions, the main theoretical result of this paper is as follows.

\begin{theorem} \label{thm:main}
Let $f$ be a normalized monotone submodular function, and let $\cA$ be a subroutine satisfying the $\beta$-iterative property.  For a given budget $k$ and parameters $2 \le \tau \le \tfrac{k}{3\eta (\log{k} + 2)}$ and $\eta \ge 4(\log{k} + 1)$, \RALG returns a set $S$ of size $k$ such that 
\begin{multline}\label{eq:main-theorem}
     f(S \setminus E^*_S) \geq \frac{\frac{\eta}{5 \beta^3 \ceillogtau + \eta} \left(1 - e^{-\frac{k - |S_0|}{\beta(k - \tau)}} \right)}{1 + \frac{\eta}{5 \beta^3 \ceillogtau + \eta} \left(1 - e^{-\frac{k - |S_0|}{\beta(k - \tau)}} \right)} \\ \times f(\opt(k, V, \tau) \setminus E^*_{\opt(k, V, \tau)}),
\end{multline}
where $E^*_S$ and $E^*_{\opt(k, V, \tau)}$ are defined as in \eqref{eq:E}.

In addition, if $\tau = o\left(\tfrac{k}{\eta \log{k}}\right)$ and $\eta \ge \log^2{k}$, then we have the following as $k \to \infty$:
\begin{multline}\label{eq:main-theorem-k-infty}
    f(S \setminus E^*_S) \geq \bigg(\frac{1 - e^{-1/\beta}}{2 - e^{-1/\beta}} + o(1)\bigg) \\ \times f(\opt(k, V, \tau) \setminus E^*_{\opt(k, V, \tau)}).
\end{multline}
In particular, \RALG-\GREEDY achieves an asymptotic approximation factor of at least $0.387$, and \RALG-\THRESH with parameter $\epsilon$ achieves an asymptotic approximation factor of at least $0.387(1-\epsilon)$.
\end{theorem}

This result solves an open problem raised in \cite{orlin2016robust}, namely, whether a constant-factor approximation guarantee can be obtained for $\tau = o(k)$ as opposed to only $\tau = o\big(\sqrt{k}\big)$.  In the asymptotic limit, our constant factor of $0.387$ for the greedy subroutine matches that of \cite{orlin2016robust}, but our algorithm permits significantly ``higher robustness'' in the sense of allowing larger $\tau$ values.  To achieve this, we require novel proof techniques, which we now outline.

\subsection{High-level overview of the analysis}

The proof of Theorem \ref{thm:main} is provided in the supplementary material.  Here we provide a high-level overview of the main challenges.  

Let $E$ denote a cardinality-$\tau$ subset of the returned set $S$ that is removed.  By the construction of the partitions, it is easy to verify that each partition $i$ contains a bucket from which at most $2^i$ items are removed.  We denote these by $B_0,\dotsc,B_{\lceil \log \tau \rceil}$, and write $E_{B_i} := E \cap B_i$.  Moreover, we define $E_0 := E \cap S_0$ and $E_1 := E \cap S_1$.

We establish the following lower bound on the final objective function value:
\begin{multline}
     f(S \setminus E)
     \ge  \max\bigg\{ f(S_0 \setminus E_0), f(S_1) - f(E_1| (S \setminus E)), \\ f\bigg(\bigcup_{i = 0}^{\ceillogtau} \left( B_i \setminus E_{B_i} \right)\bigg) \bigg\}. \label{eq:three_terms}
\end{multline}
The arguments to the first and third terms are trivially seen to be subsets of $S \setminus E$, and the second term represents the utility of the set $S_1$ subsided by the utility of the elements removed from $S_1$. 

The first two terms above are easily lower bounded by convenient expressions via submodular and the $\beta$-iterative property. The bulk of the proof is dedicated to bounding the third term.  To do this, we establish the following recursive relations with suitably-defined ``small'' values of $\alpha_j$:
\begin{gather}
    f\left(\bigcup_{i = 0}^j \left( B_i \setminus E_{B_i} \right)\right) \ge \left(1 - \frac{1}{1 + \frac{1}{\alpha_j}} \right) f(B_j) \nonumber \\
    f\left(E_{B_j}\ \given[\Big] \ \bigcup_{i = 0}^{j - 1} \left( B_i \setminus E_{B_i} \right)\right) \le \alpha_j f\left(\bigcup_{i = 0}^{j - 1} \left( B_i \setminus E_{B_i} \right)\right). \nonumber
\end{gather}
Intuitively, the first equation shows that the objective value from buckets $i=0,\dotsc,j$ {\em with removals} cannot be too much smaller than the objective value in bucket $j$ {\em without removals}, and the second equation shows that the loss in bucket $j$ due to the removals is at most a small fraction of the objective value from buckets $0,\dotsc,j-1$.  The proofs of both the base case of the induction and the inductive step make use of submodularity properties and the $\beta$-iterative property (\emph{cf.}, Definition \ref{def:beta}).

Once the suitable lower bounds are obtained for the terms in \eqref{eq:three_terms}, the analysis proceeds similarly to \cite{orlin2016robust}.  Specifically, we can show that as the second term increases, the third term decreases, and accordingly lower bound their maximum by the value obtained when the two are equal.  A similar balancing argument is then applied to the resulting term and the first term in \eqref{eq:three_terms}.

The condition $\tau \le \tfrac{k}{3\eta (\log{k} + 2)}$ follows directly from Proposition \ref{proposition:S_0_size}; namely, it is a sufficient condition for $|S_0| \le k$, as is required by \RALG.

\section{Experiments} \label{sec:experiments}

In this section, we numerically validate the performance of $\RALG$ and the claims given in the preceding sections.  In particular, we compare our algorithm against the $\OSU$ algorithm proposed in~\cite{orlin2016robust} on different datasets and corresponding objective functions (see Table~\ref{table:datasets}).  We demonstrate matching or improved performance in a broad range of settings, as well as observing that $\RALG$ can be implemented with larger values of $\tau$, corresponding to a greater robustness.  Moreover, we show that for certain real-world data sets, the classic $\GREEDY$ algorithm can perform badly for the robust problem. We do not compare against $\SAT$ \cite{krause2008robust}, due to its high computational cost for even a small $\tau$.  

{\bf Setup.} Given a solution set $S$ of size $k$, we measure the performance in terms of the minimum objective value upon the worst-case removal of $\tau$ elements, i.e. $\min_{Z \subseteq S, |Z|\leq{\tau}}\  f(S\setminus Z)$. Unfortunately, for a given solution set $S$, finding such a set $Z$ is an instance of the submodular minimization problem with a cardinality constraint,\footnote{This can be seen by noting that for submodular $f$ and any  $Z \subseteq X \subseteq V$, $f'(Z) = f(X \setminus Z)$ remains submodular.} which is known to be NP-hard with polynomial approximation factors \cite{svitkina2011submodular}. Hence, in our experiments, we only implement the optimal ``adversary'' (i.e., removal of items) for small to moderate values of $\tau$ and $k$, for which we use   a fast C++ implementation of branch-and-bound.  

Despite the difficulty in implementing the optimal adversary, we observed in our experiments that the {\em greedy adversary}, which iteratively removes elements to reduce the objective value as much as possible, has a similar impact on the objective compared to the optimal adversary for the data sets considered.  Hence, we also provide a larger-scale experiment in the presence of a greedy adversary.  Throughout, we write OA and GA to abbreviate the optimal adversary and greedy adversary, respectively.

In our experiments, the size of the robust part of the solution set (i.e.,~$|S_0|$) is set to $\tau^2$ and $\tau \log \tau$ for \OSU and \RALG, respectively.  That is, we set $\eta = 1$ in \RALG, and similarly ignore constant and logarithmic factors in \OSU, since both appear to be unnecessary in practice.  We show both the ``raw'' objective values of the solutions, as well as the objective values after the removal of $\tau$ elements.  In all experiments, we implement \GREEDY using the  \LGREEDY implementation given in \cite{minoux1978}.

 The objective functions shown in Table~\ref{table:datasets} are given in Section \ref{sec:apps}.  For the exemplar objective function, we use $d(s,v) = \| s - v \|^2$, and let the reference element $e_0$ be the zero vector. Instead of using the whole set $V$, we approximate the objective by considering a smaller random subset of $V$ for improved computational efficiency. Since the objective is additively decomposable and bounded, standard concentration bounds (e.g., the Chernoff bound) ensure that the empirical mean over a random subsample can be made arbitrarily accurate. 

{\bf Data sets.} We consider the following datasets, along with the objective functions given in Section \ref{sec:apps}:

\begin{itemize} \setlength\itemsep{1ex}
    \item \textsc{ego-Facebook}:
    This network data consists of social circles (or friends lists) from Facebook
    forming an undirected graph with $4039$ nodes and $88234$ edges.  
    \item  \textsc{ego-Twitter}: This dataset consists of $973$ social circles from Twitter, forming a directed graph with $81306$ nodes and $1768149$ edges. Both \textsc{ego-Facebook} and \textsc{ego-Twitter} were used previously in~\cite{mcauley2012}. 
    \item \textsc{Tiny10k} and \textsc{Tiny50k}: We used two Tiny Images data sets of size $10k$ and $50k$ consisting of images each represented as a $3072$-dimensional vector~\cite{torralba2008}. Besides the number of images, these two datasets also differ in the number of classes that the images are grouped into. We shift each vectors to have zero mean.
    \item \textsc{CM-Molecules}: This dataset consists of $7211$ small organic molecules, each represented as a $276$ dimensional vector. Each vector is obtained by processing the molecule's \emph{Coulomb} matrix representation \cite{rupp2015machine}. We shift and normalize each vector to zero mean and unit norm.
\end{itemize}

\begin{figure*}[ht!]
\vspace*{2ex}
    \centering
    \begin{subfigure}{.33 \textwidth}
      \centering
      \includegraphics[scale=0.3]{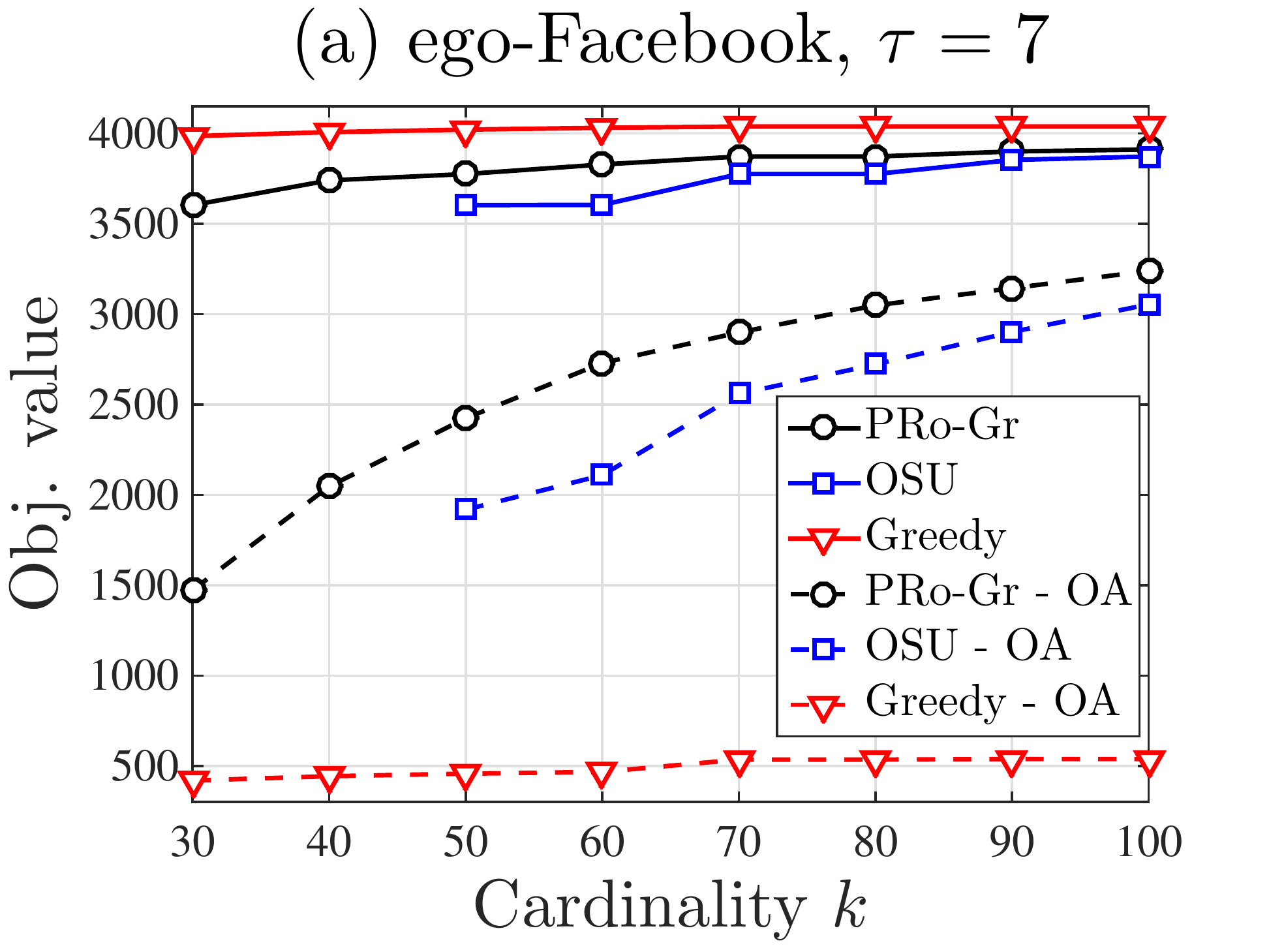}
    \end{subfigure}
    \begin{subfigure}{.33\textwidth}
      \centering
      \includegraphics[scale=0.3]{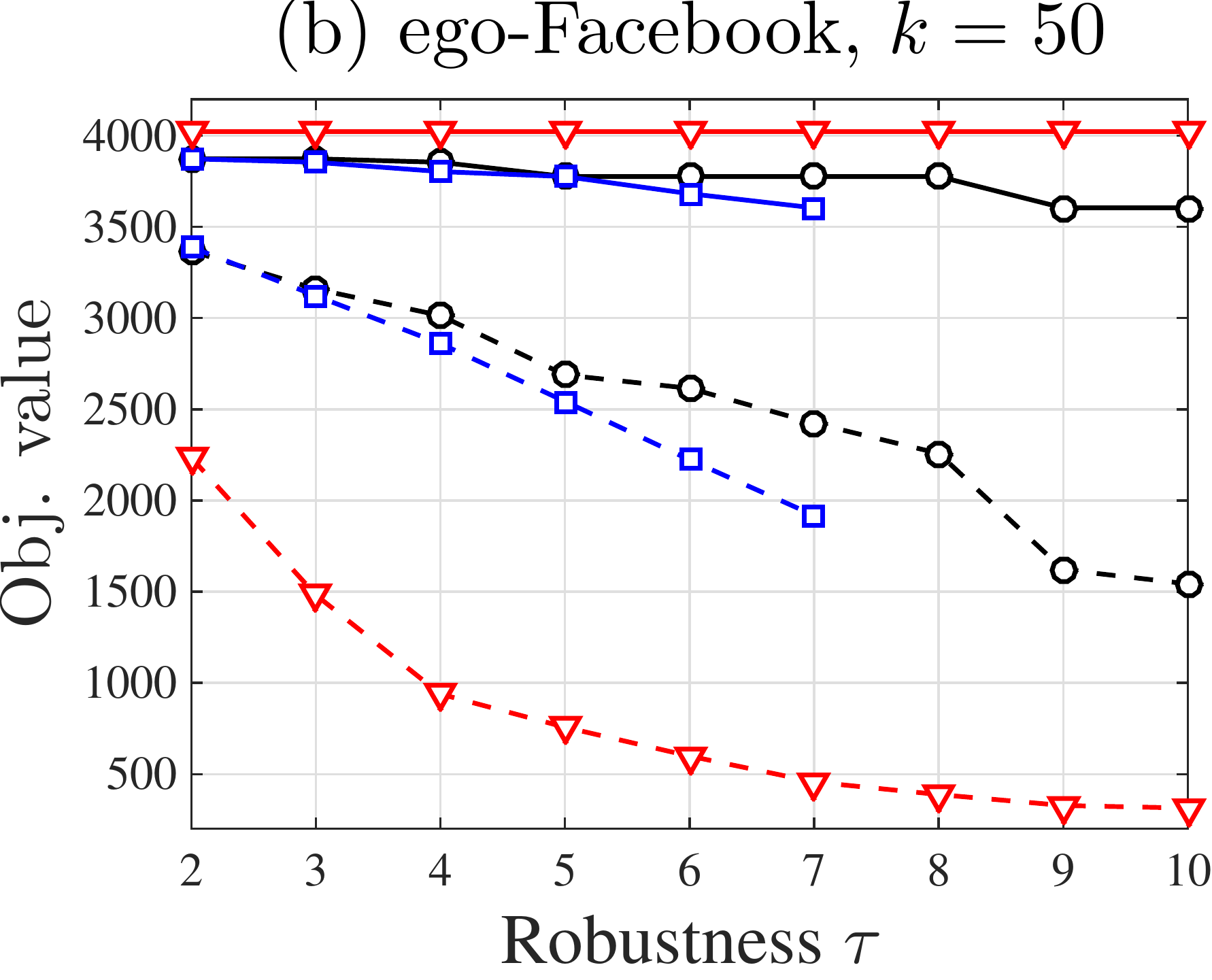}
    \end{subfigure}%
    \begin{subfigure}{.33\textwidth}
      \centering
      \includegraphics[scale=0.3]{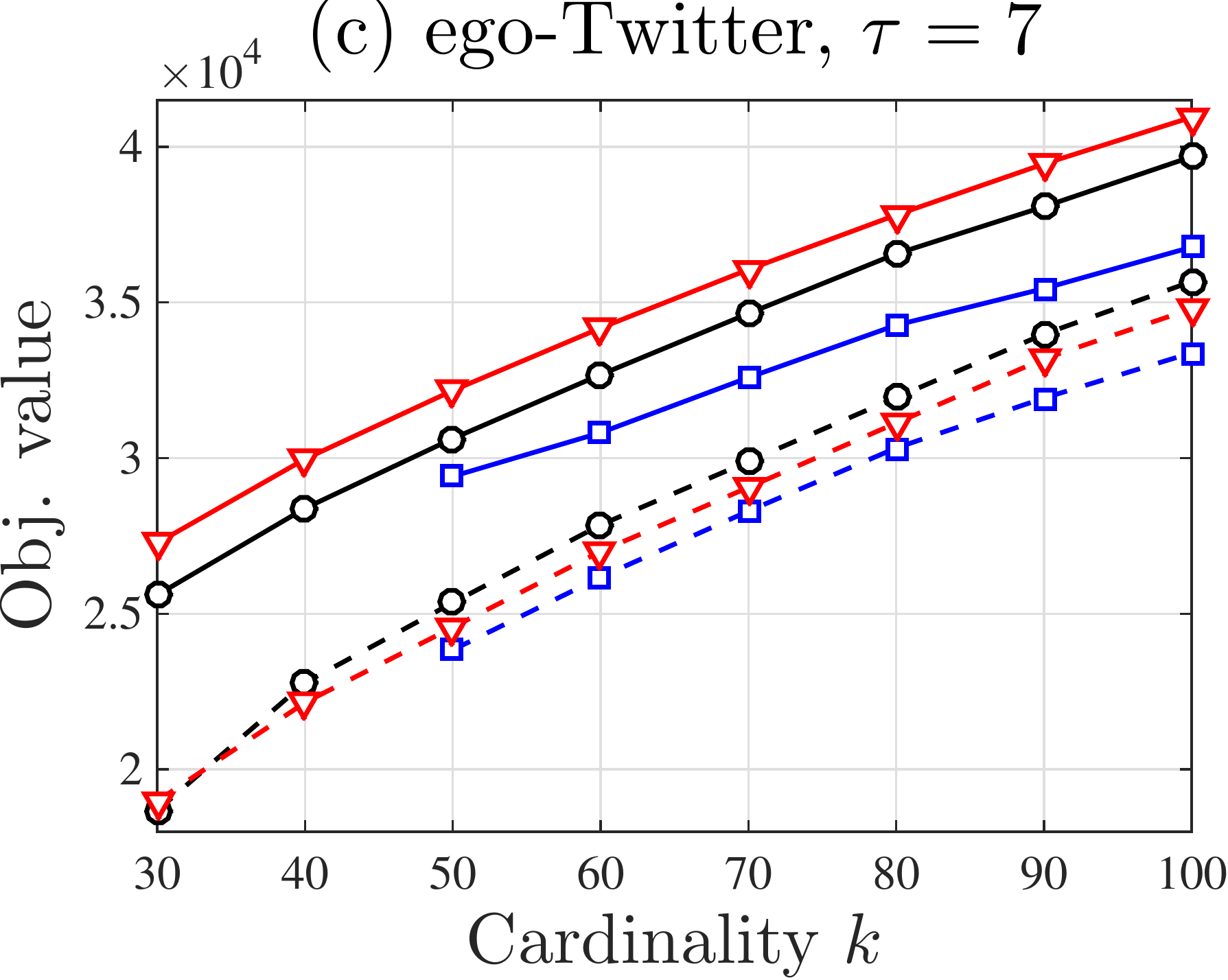}
    \end{subfigure}
    
    \medskip
    
    \begin{subfigure}{.33 \textwidth}
      \centering
      \includegraphics[scale=0.3]{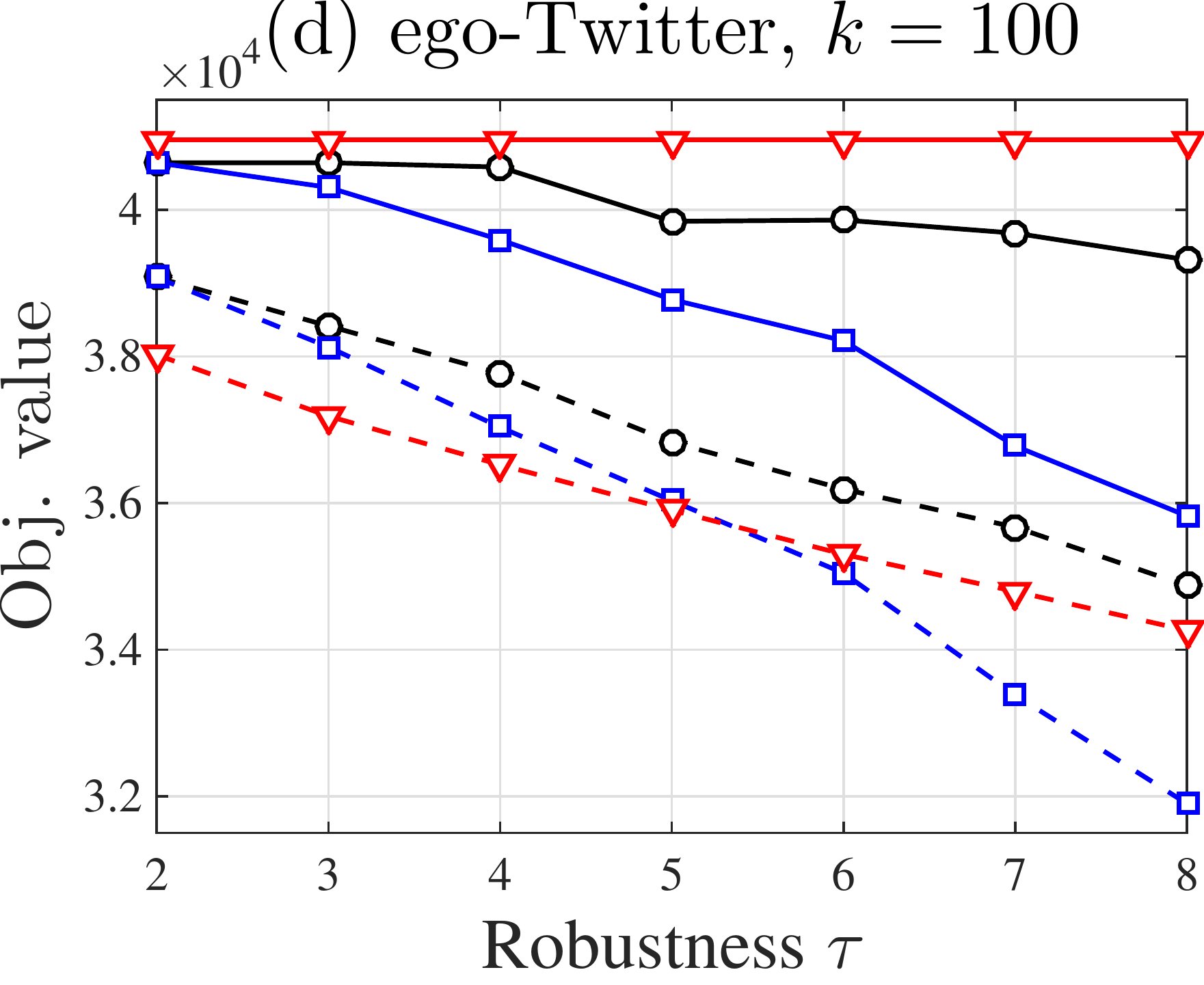}
    \end{subfigure}
    \begin{subfigure}{.33\textwidth}
      \centering
      \includegraphics[scale=0.3]{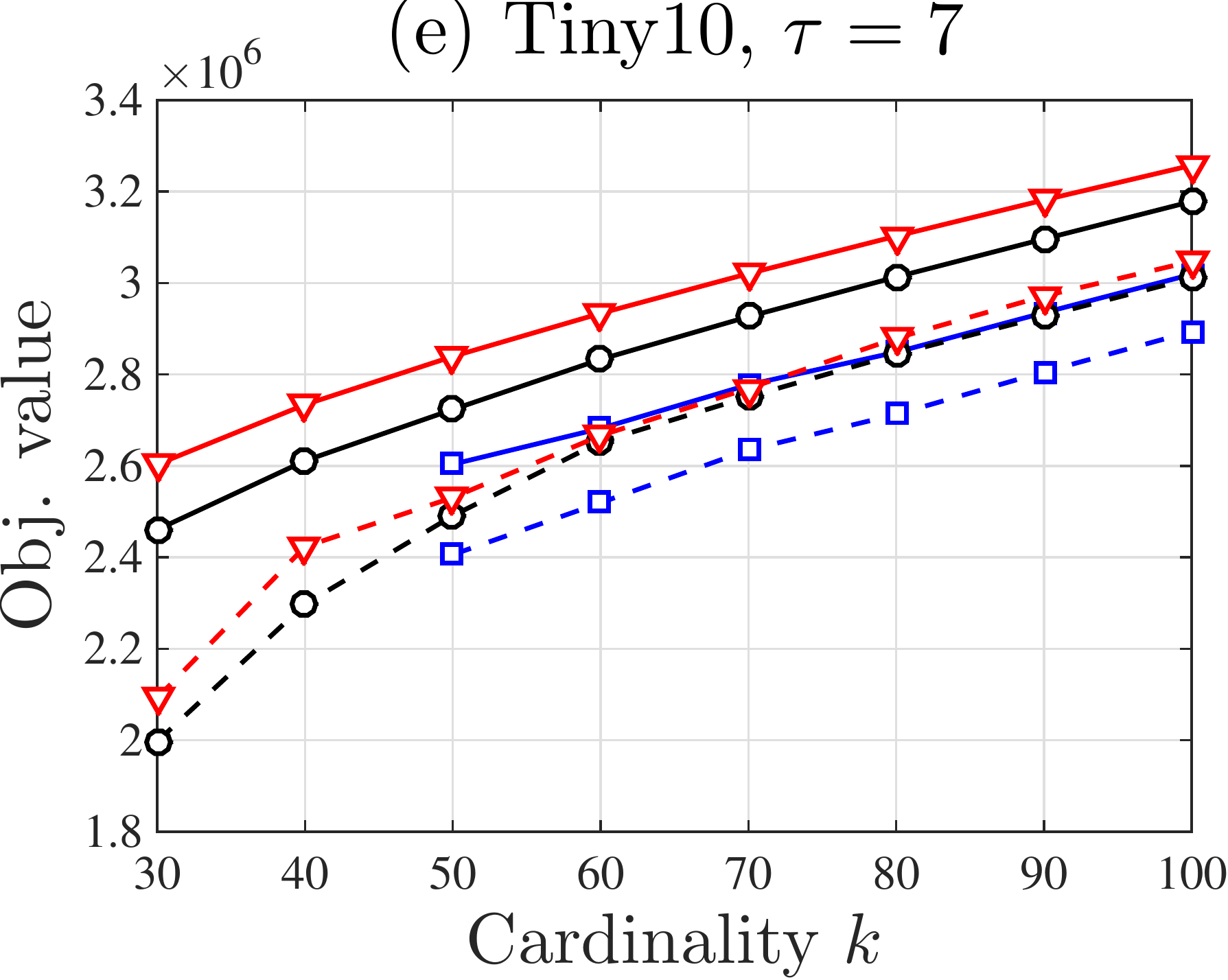}
    \end{subfigure}%
    \begin{subfigure}{.33\textwidth}
      \centering
      \includegraphics[scale=0.3]{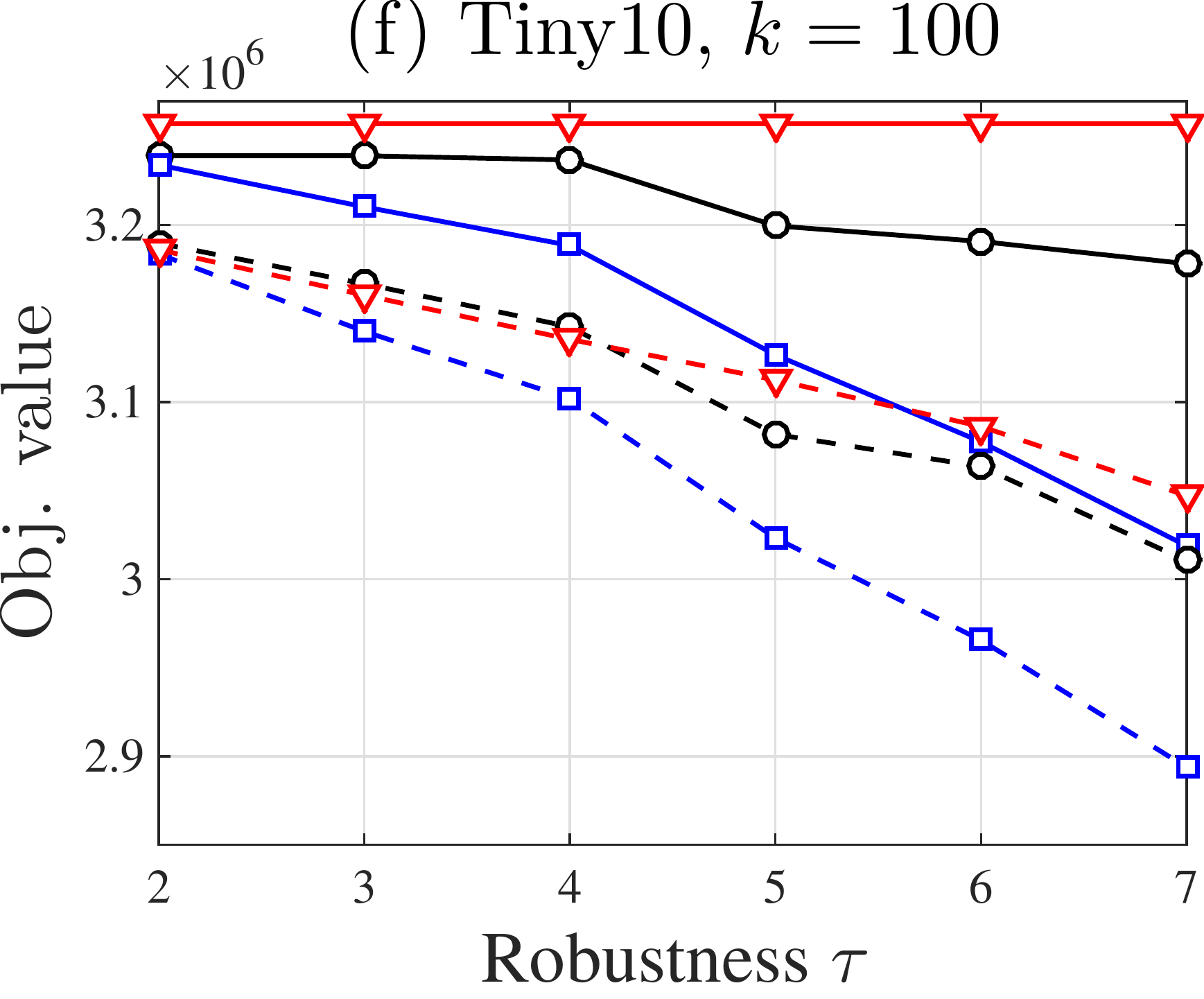}
    \end{subfigure}

    \medskip
    
    \begin{subfigure}{.33 \textwidth}
      \centering
      \includegraphics[scale=0.3]{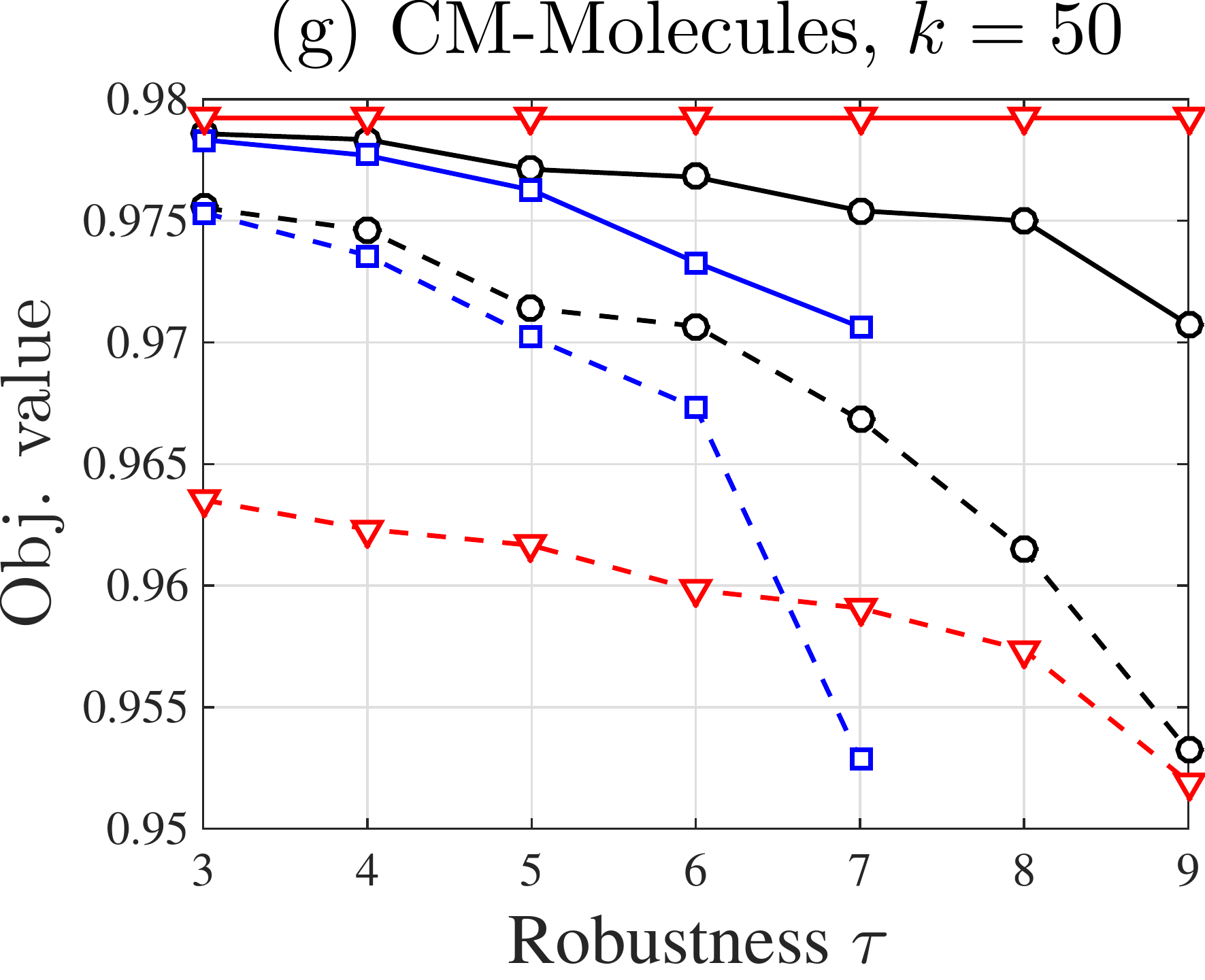}
    \end{subfigure}
    \begin{subfigure}{.33\textwidth}
      \centering
      \includegraphics[scale=0.3]{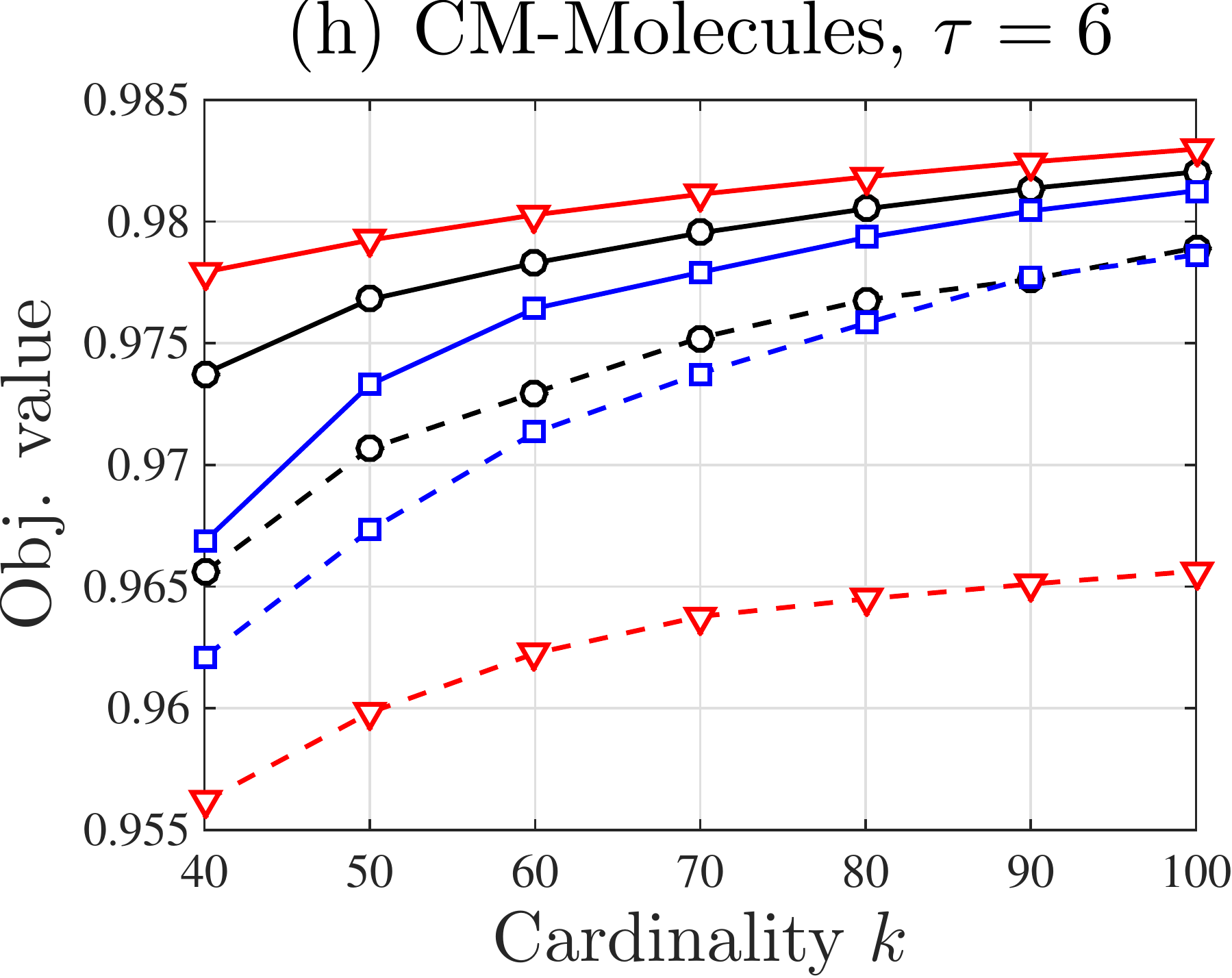}
    \end{subfigure}%
    \begin{subfigure}{.33\textwidth}
      \centering
      \includegraphics[scale=0.3]{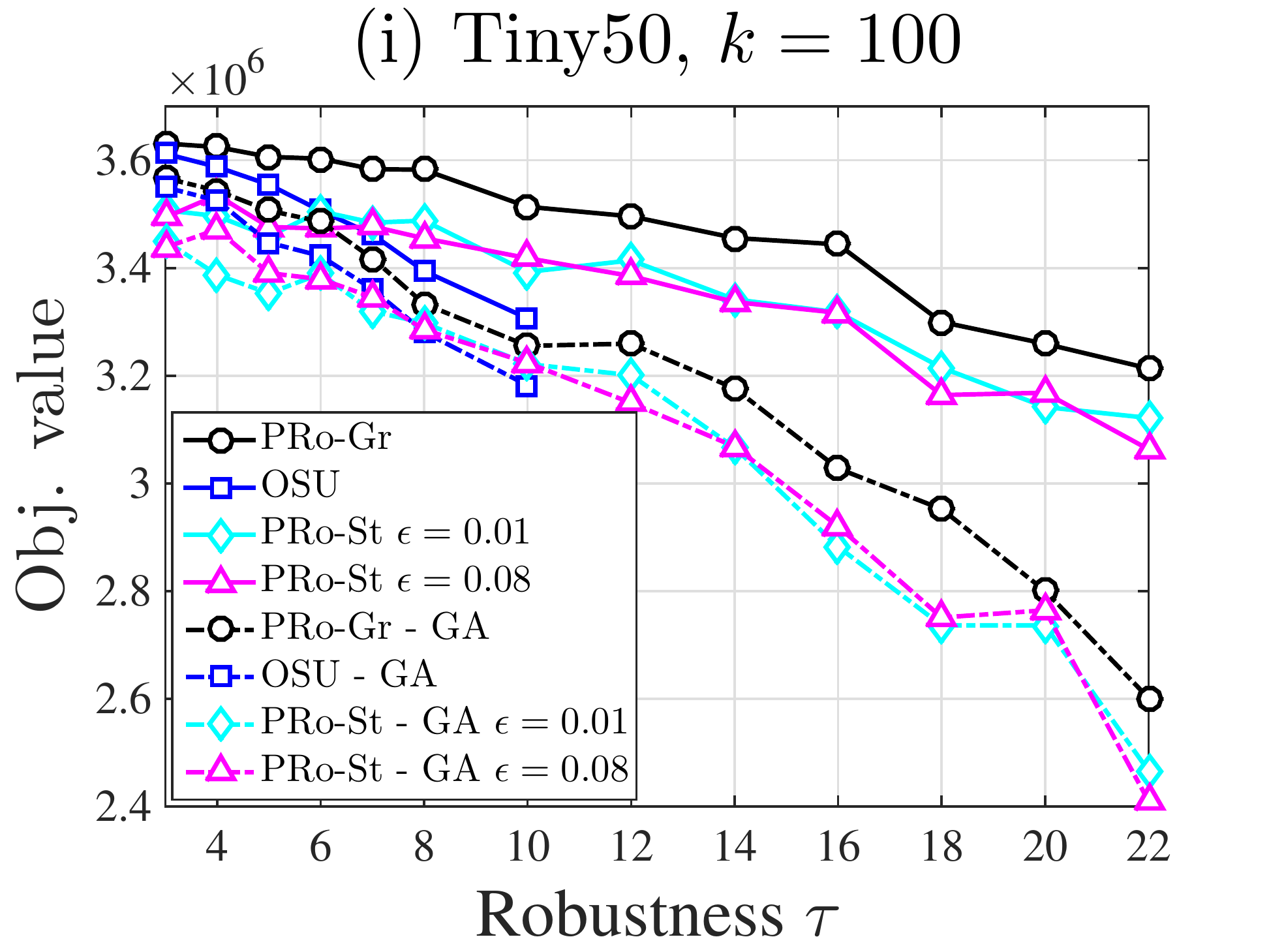}
    \end{subfigure}
    \vspace*{1ex}
    \caption{Numerical comparisons of the algorithms \RALG-\GREEDY, \GREEDY and \OSU, and their objective values {$\RALG$\textsc{-OA}}, {$\OSU$\textsc{-OA}} and {$\GREEDY$\textsc{-OA}} once $\tau$ elements are removed. Figure (i) shows the performance on the larger scale experiment where both \GREEDY and \STOCHG are used as subroutines in \RALG.}  
    \vspace*{2ex}
\end{figure*}
\begin{table}[h!]
    \centering
    \vspace*{2ex}
    \begin{tabular}{cccc} \toprule
        {Dataset} & $n$ & dimension & $f$  \\ \midrule
        Tiny-10k & $10\ 000$ & $3074$ &  Exemplar \\
        Tiny-50k  & $50\ 000$ & $3074$ & Exemplar\\
        CM-Molecules  & $7211$ & $276$ & Exemplar \\ \midrule
        {Network} & \# nodes & \# edges &  $f$ \\ \midrule
        ego-Facebook &  $4039$ & $88\ 234$  &  DomSet \\
        ego-Twitter  &  $81\ 306$ &  $1\ 768\ 149$ &  DomSet \\
        \bottomrule
    \end{tabular}
    \caption{Datasets and corresponding objective functions.}
    \label{table:datasets}
\end{table}

{\bf Results.} In the first set of experiments, we compare $\RALG$-$\GREEDY$ (written using the shorthand $\GRALG$ in the legend) against $\GREEDY$ and $\OSU$ on the \textsc{ego-Facebook} and \textsc{ego-Twitter} datasets. In this experiment, the dominating set selection objective in \eqref{eq:domset} is considered. Figure~2 (a) and (c) show the results before and after the worst-case removal of $\tau = 7$ elements for different values of $k$. In Figure 2 (b) and (d), we show the objective value for fixed $k=50$ and $k=100$, respectively, while the robustness parameter $\tau$ is varied. 

\GREEDY~achieves the highest raw objective value, followed by \RALG-\GREEDY~and \OSU. However, after the worst-case removal, {$\RALG$-\GREEDY-\textsc{OA}} outperforms both {$\OSU$\textsc{-OA}} and {$\GREEDY$\textsc{-OA}}. In Figure 2 (a) and (b), {$\GREEDY$\textsc{-OA}} performs poorly due to a high concentration of the objective value on the first few elements selected by $\GREEDY$. While $\OSU$ requires $k \geq \tau^2$, $\RALG$ only requires $k \geq \tau \log \tau$, and hence it can be run for larger values of $\tau$ (e.g., see Figure 2 (b) and (c)). Moreover, in Figure 2 (a) and (b), we can observe that although \RALG uses a smaller number of elements to build the robust part of the solution set, it has better robustness in comparison with \OSU. 

In the second set of experiments, we perform the same type of comparisons on the \textsc{Tiny10} and \textsc{CM-Molecules} datasets. The exemplar based clustering in \eqref{eq:exemplar} is used as the objective function. In Figure 2 (e) and (h), the robustness parameter is fixed to $\tau = 7$ and $\tau=6$, respectively, while the cardinality $k$ is varied. In Figure 2 (f) and (h), the cardinality is fixed to $k=100$ and $k=50$, respectively, while the robustness parameter $\tau$ is varied.

Again, \GREEDY~achieves the highest objective value. On the \textsc{Tiny10} dataset, {$\GREEDY$\textsc{-OA}} (Figure 2 (e) and (f)) has a large gap between the raw and final objective, but it still slightly outperforms {$\RALG$-\GREEDY-\textsc{OA}}.  This demonstrates that \GREEDY can work well in some cases, despite failing in others.  We observed that it succeeds here because the objective value is relatively more uniformly spread across the selected elements. On the same dataset, {$\RALG$-\GREEDY-\textsc{OA}} outperforms {$\OSU$\textsc{-OA}}. On our second dataset \textsc{CM-Molecules} (Figure 2 (g) and (h)), {$\RALG$-\GREEDY-\textsc{OA}} achieves the highest robust objective value, followed by {$\OSU$\textsc{-OA}} and {$\GREEDY$\textsc{-OA}}.

In our final experiment (see Figure 2 (i)), we compare the performance of \RALG-\GREEDY against two instances of \RALG-\STOCHG  with $\epsilon=0.01$ and $\epsilon =0.08$ (shortened to \SRALG in the legend), seeking to understand to what extent using the more efficient stochastic subroutine impacts the performance.  We also show the performance of \OSU. In this experiment, we fix $k=100$ and vary $\tau$. We use the greedy adversary instead of the optimal one, since the latter becomes computationally challenging for larger values of $\tau$.

In Figure 2 (i), we observe a slight decrease in the objective value of  \RALG-\STOCHG due to the stochastic optimization. On the other hand, the gaps between the robust and non-robust solutions remain similar, or even shrink.  Overall, we observe that at least in this example, the stochastic subroutine does not compromise the quality of the solution too significantly, despite having a lower computational complexity.
 


    
    
      
    

\section{Conclusion} 
We have provided a new Partitioned Robust (\RALG) submodular maximization algorithm attaining a constant-factor approximation guarantee for general $\tau = o(k)$, thus resolving an open problem posed in \cite{orlin2016robust}.  Our algorithm uses a novel partitioning structure with partitions consisting of buckets with exponentially decreasing size, thus providing a ``robust part'' of size $O(\tau \mathrm{poly}\log \tau)$.  We have presented a variety of numerical experiments where \RALG outperforms both \GREEDY and \OSU. A potentially interesting direction for further research is to understand the {\em linear regime}, in which $\tau = \alpha k$ for some constant $\alpha \in (0,1)$, and in particular, to seek a constant-factor guarantee for this regime.

{\bf Acknowledgment.} This work was supported in part by the European Commission under Grant ERC Future Proof, SNF 200021-146750 and SNF CRSII2-147633, and `EPFL Fellows' (Horizon2020 665667).

 
\newpage
\bibliography{cite}
\bibliographystyle{icml2017}

\onecolumn
{\huge \bf \centering Supplementary Material \par }

{\large \centering {\bf ``Robust Submodular Maximization: A Non-Uniform Partitioning Approach'' (ICML 2017)} \\ Ilija Bogunovic, Slobodan Mitrovi\'c, Jonathan Scarlett, and Volkan Cevher \par}
\appendix
\section{Proof of Proposition \ref{proposition:S_0_size}}
We have
 \begin{align*}
 |S_0| &= \sum_{i=0}^{\ceillogtau} \lceil \tau / 2^i \rceil 2^i \eta \\
       &\le  \sum_{i=0}^{\ceillogtau} \left( \frac{\tau}{2^i} + 1 \right) 2^i \eta \\
       &\le \eta(\ceillogtau + 1)(\tau + 2^{\ceillogtau})\\
       &\le 3\eta\tau(\ceillogtau + 1) \\
       &\le 3\eta\tau(\log k + 2).
 \end{align*}

\section{Proof of Proposition~\ref{proposition:beta-nice-main-rule}}
Recalling that $\cA_j(T)$ denotes a set constructed by the algorithm after $j$ iterations, we have
\begin{align}
    f(\cA_{j}(T)) - f(\cA_{j-1}(T)) &\geq \frac{1}{\beta} \max_{e \in T} f(e|\cA_{j-1}(T)) \nonumber\\
    &\geq \frac{1}{\beta} \max_{e \in T} f(e|\cA_{k}(T)) \nonumber\\ 
    &\geq \frac{1}{\beta} \max_{e \in T \setminus \cA_k(T)} f(e|\cA_{k}(T)), \label{eq:beta_prop_1}
\end{align}
where the first inequality follows from the $\beta$-iterative property \eqref{eq:beta-property}, and the second inequality follows from  $\cA_{j-1}(S) \subseteq \cA_{k}(S)$ and the submodularity of $f$.

Continuing, we have 
\begin{align*}
    f(\cA_{k}(T)) &= \sum_{j=1}^{k} f(\cA_{j}(T)) - f(\cA_{j-1}(T)) \\
    & \geq  \frac{k}{\beta} \max_{e \in T \setminus \cA_k(T)} f(e| \cA_{k}(T)),
\end{align*}
where the last inequality follows from \eqref{eq:beta_prop_1}.

By rearranging, we have for any $e \in T \setminus \cA_k(T) $ that
\[
    f(e| \cA_k(T)) \leq \beta \frac{f(\cA_k(T))}{k}.
\]

\section{Proof of Lemma \ref{lemma:l-k-beta}}

Recalling that $A_j(T)$ denotes the set constructed after $j$ iterations when applied to $T$, we have
\begin{align}
    \max_{e \in T \setminus A_{j-1}(T)} f(e|A_{j-1}(T)) & \geq \frac{1}{k} \sum_{e \in \opt(k,T) \setminus A_{j-1}(T)} f(e|A_{j-1}(T)) \nonumber \\
    & \geq \frac{1}{k} f(\opt(k,T)| A_{j-1}(T)) \nonumber \\
    & \geq \frac{1}{k} \big( f(\opt(k,T))  - f(A_{j-1}(T))\big), \label{eq:beta-l-k_1}
\end{align}
where the first line holds since the maximum is lower bounded by the average, the line uses submodularity, and the last line uses monotonicity.

By combining the $\beta$-iterative property with \eqref{eq:beta-l-k_1}, we obtain
\begin{align*}
    f(\cA_{j}(T)) - f(\cA_{j-1}(T)) &\geq \frac{1}{\beta} \max_{e \in T \setminus A_{j-1}(T)} f(e|A_{j-1}(T)) \\
    & \geq \frac{1}{k \beta} \big( f(\opt(k, T)) - f(A_{j-1}(T)) \big).
\end{align*}
By rearranging, we obtain
\begin{equation}
    \label{eq:beta-l-k_2}
    f(\opt(k, T)) - f(A_{j-1}(T)) \leq  \beta k \big(  f(\cA_{j}(T)) - f(\cA_{j-1}(T)) \big).
\end{equation}
We proceed by following the steps from the proof of Theorem~1.5 in \cite{krause2012submodular}. Defining $\delta_{j}:= f(\opt(k, T)) - f(A_{j}(T))$, we can rewrite~\eqref{eq:beta-l-k_2} as $\delta_{j-1} \leq \beta k (\delta_{j-1} - \delta_{j})$.
By rearranging, we obtain 
\[
    \delta_j \leq \left(1 - \frac{1}{\beta k} \right) \delta_{j-1}.
\]
Applying this recursively, we obtain $\delta_l \leq \big( 1 - \frac{1}{\beta k}\big)^l \delta_0$, where $\delta_0 = f(\opt(k, T))$ since $f$ is normalized (i.e., $f(\emptyset) = 0$). Finally, applying $1 - x \leq e^{-x}$ and rearranging, we obtain
\[
    f(\cA_{l}(T)) \geq \left( 1 - e^{-\frac{l}{\beta k}} \right) f(\opt(k,T)).
\]

\section{Proof of Theorem \ref{thm:main}}

\subsection{Technical Lemmas}

We first provide several technical lemmas that will be used throughout the proof.  We begin with a simple property of submodular functions.

 \begin{lemma}\label{lemma:A-B-R}
     For any submodular function $f$ on a ground set $V$, and any sets $A,B,R \subseteq V$, we have
     \[
         f(A \cup B) - f(A \cup (B \setminus R)) \le f(R\ |\ A).
     \]
 \end{lemma}

 \begin{proof}
     Define $R_2 := A \cap R$, and $R_1 := R \setminus A = R \setminus R_2$. We have
 \begin{align}
       f(A \cup B) - f(A \cup (B \setminus R))
     & =  f(A \cup B) - f((A \cup B) \setminus R_1) \nonumber \\
     & =  f(R_1\ |\ (A \cup B) \setminus R_1) \nonumber\\
     & \le  f(R_1\ |\ (A \setminus R_1)) \label{eq:ABR_1} \\
     & =  f(R_1\ |\ A) \label{eq:ABR_2} \\
     & =  f(R_1 \cup R_2\ |\ A) \label{eq:ABR_3} \\
     & =  f(R\ |\ A) \nonumber ,
 \end{align}    
 where~\eqref{eq:ABR_1} follows from the submodularity of $f$, \eqref{eq:ABR_2} follows since $A$ and $R_1$ are disjoint, and \eqref{eq:ABR_3} follows since $R_2 \subseteq A$.
 \end{proof}
 
The next lemma provides a simple lower bound on the maximum of two quantities; it is stated formally since it will be used on multiple occasions.
 
  \begin{lemma}
    \label{lemma:A-B-max}
    For any set function $f$, sets $A,B$, and constant $\alpha > 0$, we have
    \begin{equation}
        \label{eq:A-B-max_1}
         \max \lbrace f(A), f(B) - \alpha f(A) \rbrace \geq \left( \frac{1}{1+\alpha} \right) f(B),
    \end{equation}
    and
    \begin{equation}
        \label{eq:A-B-max_2}
         \max \lbrace \alpha f(A), f(B) - f(A) \rbrace \geq \left( \frac{\alpha}{1+\alpha} \right) f(B).
    \end{equation}

 \end{lemma}
 \begin{proof}
    Starting with \eqref{eq:A-B-max_1}, we observe that one term is increasing in $f(A)$ and the other is decreasing in $f(A)$.  Hence, the maximum over all possible $f(A)$ is achieved when the two terms are equal, i.e., $f(A)  =  \frac{1}{1+\alpha} f(B)$.  We obtain \eqref{eq:A-B-max_2} via the same argument.
\end{proof}

The following lemma relates the function values associated with two buckets formed by $\RALG$, denoted by $X$ and $Y$.  It is stated with respect to an arbitrary set $E_Y$, but when we apply the lemma, this will correspond to the elements of $Y$ that are removed by the adversary.

 \begin{lemma}\label{lemma:base-case}
     Under the setup of Theorem \ref{thm:main}, let $X$ and $Y$ be buckets of $\RALG$ such that $Y$ is constructed at a later time than $X$. For any set $E_Y \subseteq Y$, we have
     \[
         f(X \cup (Y \setminus E_Y)) \ge \frac{1}{1 + \alpha}f(Y),
     \]
     and
     \begin{equation}\label{eq:base-case-Ey-X}
         f(E_Y\ |\ X) \le \alpha f(X),
     \end{equation}
     where $\alpha = \beta \tfrac{|E_Y|}{|X|}$.
 \end{lemma}
 \begin{proof}
 Inequality~\eqref{eq:base-case-Ey-X} follows from the $\beta$-iterative property of $\cA$; specifically, we have from~\eqref{eq:beta-nice-2} that
 \[
    f(e|X) \leq \beta \frac{f(X)}{|X|},
 \]
 where $e$ is any element of the ground set that is neither in $X$ nor any bucket constructed before $X$.
 Hence, we can write
\[
    f(E_Y\ |\ X) \leq \sum_{e \in E_Y} f(e|X) \leq \beta \frac{|E_Y|}{|X|} f(X) = \alpha f(X),
\]
where the first inequality is by submodularity.  This proves \eqref{eq:base-case-Ey-X}.

Next, we write
 \begin{align}
     f(Y) - f(X \cup (Y \setminus E_Y))
     &\le  f(X \cup Y) - f(X \cup (Y \setminus E_Y)) \label{eq:Y-X0} \\
     &\le  f(E_Y\ |\ X), \label{eq:Y-X}
 \end{align}
 where \eqref{eq:Y-X0} is by monotonicity, and \eqref{eq:Y-X} is by Lemma~\ref{lemma:A-B-R} with $A = X$, $B = Y$, and $R = E_Y$.

Combining  \eqref{eq:base-case-Ey-X} and \eqref{eq:Y-X}, together with the fact that $f(X \cup (Y \setminus E_Y)) \geq  f(X)$ (by monotonicity), we have
 \begin{align}
     f(X \cup (Y \setminus E_Y)) &\geq \max \left\lbrace f(X), f(Y) - \alpha f(X) \right\rbrace \nonumber \\
     &\ge \frac{1}{1 + \alpha} f(Y), \label{eq:X-Y-Ey-1}
 \end{align}
 where~\eqref{eq:X-Y-Ey-1} follows from \eqref{eq:A-B-max_1}.
 \end{proof}
 
Finally, we provide a lemma that will later be used to take two bounds that are known regarding the previously-constructed buckets, and use them to infer bounds regarding the next bucket. 

 \begin{lemma}\label{lemma:final}
     Under the setup of Theorem \ref{thm:main}, let $Y$ and $Z$ be buckets of $\RALG$ such that $Z$ is constructed at a later time than $Y$, and let $E_Y \subseteq Y$ and $E_Z \subseteq Z$ be arbitrary sets.  Moreover, let $X$ be a set (not necessarily a bucket) such that
     \begin{equation}\label{lemma:final-condition1}
         f((Y \setminus E_Y) \cup X) \ge \frac{1}{1 + \alpha} f(Y),
     \end{equation}
     and
     \begin{equation}\label{lemma:final-condition2}
         f(E_Y\ |\ X) \le \alpha f(X).
     \end{equation}
     Then, we have
     \begin{equation}\label{eq:bound-on-Ez}
         f(E_Z\ |\ (Y \setminus E_Y) \cup X) \le \alphan f((Y \setminus E_Y) \cup X),
     \end{equation}
     and
     \begin{equation}\label{eq:bound-on-Z--Ez}
         f((Z \setminus E_Z) \cup (Y \setminus E_Y) \cup X) \ge \frac{1}{1 + \alphan} f(Z),
     \end{equation}
     where 
     \begin{equation}
         \alphan = \beta \frac{|E_Z|}{|Y|} (1 + \alpha) + \alpha.
     \end{equation}
 \end{lemma}
 \begin{proof}
     We first prove~\eqref{eq:bound-on-Ez}:
     \begin{align}
         f(E_Z\ |\ (Y \setminus E_Y) \cup X)
         & = f((Y \setminus E_Y) \cup X \cup E_Z) - f((Y \setminus E_Y) \cup X) \nonumber \\
         & \le  f(X \cup Y \cup E_Z) - f((Y \setminus E_Y) \cup X) \label{eq:Ez-bound-1} \\
         & = f(E_Z\ |\ X \cup Y) + f(X \cup Y) - f((Y \setminus E_Y) \cup X) \nonumber \\
         & \le  f(E_Z\ |\ Y) + f(X \cup Y) - f((Y \setminus E_Y) \cup X) \label{eq:Ez-bound-2} \\
         & \le  \beta \frac{|E_Z|}{|Y|} f(Y) + f(X \cup Y) - f((Y \setminus E_Y) \cup X) \label{eq:Ez-bound-3}\\
         & \le \beta \frac{|E_Z|}{|Y|} (1 + \alpha) f((Y \setminus E_Y) \cup X) + f(X \cup Y) - f((Y \setminus E_Y) \cup X) \label{eq:Ez-bound-4}  \\
         & \le \beta \frac{|E_Z|}{|Y|} (1 + \alpha) f((Y \setminus E_Y) \cup X) + f(E_Y\ |\ (Y \setminus E_Y) \cup X) \label{eq:Ez-bound-5} \\
         & \le \beta  \frac{|E_Z|}{|Y|} (1 + \alpha) f((Y \setminus E_Y) \cup X) + f(E_Y\ |\ X) \label{eq:Ey-from-X-and-Y} \\
         & \le \beta \frac{|E_Z|}{|Y|} (1 + \alpha) f((Y \setminus E_Y) \cup X) + \alpha f(X)  \label{eq:Ez-bound-6} \\
         & \le  \beta \frac{|E_Z|}{|Y|} (1 + \alpha) f((Y \setminus E_Y) \cup X) + \alpha f((Y \setminus E_Y) \cup X) \label{eq:Ez-bound-7} \\
         & =  \left( \beta \frac{|E_Z|}{|Y|} (1 + \alpha) + \alpha \right) f((Y \setminus E_Y) \cup X).\label{eq:Ez-bound},
     \end{align}
    where:~\eqref{eq:Ez-bound-1} and~\eqref{eq:Ez-bound-2} follow by monotonicity and submodularity, respectively; \eqref{eq:Ez-bound-3} follows from the second part of Lemma~\ref{lemma:base-case}; \eqref{eq:Ez-bound-4} follows from \eqref{lemma:final-condition1};
    \eqref{eq:Ez-bound-5} is obtained by applying Lemma~\ref{lemma:A-B-R} for $A = X$, $B = Y$, and $R = E_Y$;
    \eqref{eq:Ey-from-X-and-Y} follows by submodularity;
    \eqref{eq:Ez-bound-6} follows from~\eqref{lemma:final-condition2};
    \eqref{eq:Ez-bound-7} follows by monotonicity.
    Finally, by defining $\alphan := \beta \frac{|E_Z|}{|Y|} (1 + \alpha) + \alpha$ in~\eqref{eq:Ez-bound} we establish the bound in~\eqref{eq:bound-on-Ez}.
    
     In the rest of the proof, we show that~\eqref{eq:bound-on-Z--Ez} holds as well. First, we have
     \begin{equation}\label{eq:f-Z-Ez}
         f((Z \setminus E_Z) \cup (Y \setminus E_Y) \cup X) \ge f(Z) - f(E_Z\ |\ (Y \setminus E_Y) \cup X)
     \end{equation}
     by Lemma~\ref{lemma:A-B-R} with $B=Z$, $R=E_Z$ and $A= (Y\setminus E_Y) \cup X$. Now we can use the derived bounds \eqref{eq:Ez-bound} and \eqref{eq:f-Z-Ez} to obtain
     \begin{align*}
        f((Z \setminus E_Z) \cup (Y \setminus E_Y) \cup X)
         & \ge f(Z) - f(E_Z\ |\ (Y \setminus E_Y) \cup X) \\
         & \ge f(Z) - \left(\beta \frac{|E_Z|}{|Y|} (1 + \alpha) + \alpha \right) f((Y \setminus E_Y) \cup X).
     \end{align*}
    Finally, we have
    \begin{align*}
        f((Z \setminus E_Z) \cup (Y \setminus E_Y) \cup X)
         & \geq  \max \left\lbrace f((Y \setminus E_Y) \cup X), f(Z) - \left( \beta \frac{|E_Z|}{|Y|} (1 + \alpha) + \alpha \right) f((Y \setminus E_Y) \cup X) \right\rbrace \\
         & \geq  \frac{1}{ 1 + \alphan} f(Z),
    \end{align*}
    where the last inequality follows from Lemma~\ref{lemma:A-B-R}.
 \end{proof}
 
  Observe that the results we obtain on $f(E_Z\ |\ (Y \setminus E_Y) \cup X)$ and on $f((Z \setminus E_Z) \cup (Y \setminus E_Y) \cup X)$ in Lemma~\ref{lemma:final} are of the same form as the pre-conditions of the lemma.  This will allow us to apply the lemma recursively.

\subsection{Characterizing the Adversary}

Let $E$ denote a set of elements removed by an adversary, where $|E| \le \tau$. Within $S_0$, $\RALG$ constructs $\lceil \log \tau \rceil + 1$ partitions. 
Each partition $i \in \lbrace 0, \ldots, \lceil \log \tau \rceil \rbrace$ consists of $\lceil \tau / 2^i \rceil$ buckets, each of size $ 2^i \eta $, where $\eta \in \mathbb{N}$ will be specified later. We let $B$ denote a generic bucket, and define $E_B$ to be all the elements removed from this bucket, i.e. $E_B = B \cap E$.

The following lemma identifies a bucket in each partition for which not too many elements are removed.

\begin{lemma}
\label{lemma:E-size}
    Under the setup of Theorem \ref{thm:main}, suppose that an adversary removes a set $E$ of size at most $\tau$ from the set $S$ constructed by $\RALG$.  Then for each partition $i$, there exists a bucket $B_i$ such that $|E_{B_i}| \leq 2^i$, i.e., at most $2^i$ elements are removed from this bucket.
\end{lemma}

\begin{proof}
 Towards contradiction, assume that this is not the case, i.e., assume $|E_{B_i}| > 2^i$ for every bucket of the $i$-th partition. As the number of buckets in partition $i$ is $\lceil \tau / 2^i  \rceil$, this implies that the adversary has to spend a budget of
\[
    |E| \ge 2^i |E_{B_i}| > 2^i \lceil \tau / 2^i \rceil = \tau,
\]
which is in contradiction with $|E| \leq \tau$.
\end{proof}


We consider $B_0, \ldots, B_{\lceil \log{\tau} \rceil}$ as above, and show that even in the worst case, $f\left(\bigcup_{i = 0}^{\lceil \log{\tau} \rceil} \left( B_i \setminus E_{B_i} \right)\right)$ is almost as large as $f\left(B_{\lceil \log{\tau} \rceil} \right)$ for appropriately set $\eta$. To achieve this, we apply Lemma~\ref{lemma:final} multiple times, as illustrated in the following lemma. We henceforth write $\eta_h := \eta/2$ for brevity.
\begin{lemma}\label{lemma:bound-on-B-log-tau}
    Under the setup of Theorem \ref{thm:main}, suppose that an adversary removes a set $E$ of size at most $\tau$ from the set $S$ constructed by $\RALG$, and let $B_0, \cdots, B_{\ceillogtau}$ be buckets such that $|E_{B_i}| \leq 2^i$ for each $i \in \lbrace 1, \cdots \ceillogtau \rbrace$ (\emph{cf.}, Lemma \ref{lemma:E-size}). Then,
    \begin{equation}\label{eq:bound-on-B-log-tau}
        f\left(\bigcup_{i = 0}^{\ceillogtau} \left( B_i \setminus E_{B_i} \right)\right) \ge \left(1 - \frac{1}{1 + \frac{1}{\alpha}} \right) f\left(B_{\ceillogtau} \right) = \frac{1}{1 +\alpha}f\left(B_{\ceillogtau} \right),
    \end{equation}
    and
    \begin{equation} \label{eq:second_cond}
        f\left(E_{B_\ceillogtau}\ \given[\Big] \ \bigcup_{i = 0}^{\ceillogtau - 1} \left( B_i \setminus E_{B_i} \right)\right) \le \alpha f\left(\bigcup_{i = 0}^{\ceillogtau - 1} \left( B_i \setminus E_{B_i} \right)\right),
    \end{equation}
    for some
    \begin{equation}\label{eq:alpha_log-tau-bound}
        \alpha \le \beta^2 \frac{(1 + \eta_h)^{\ceillogtau} - \eta_h^{\ceillogtau}}{\eta_h^{\ceillogtau}}.
    \end{equation}
\end{lemma}
\begin{proof}
    In what follows, we focus on the case where there exists some bucket $B_0$ in partition $i=0$ such that $B_0 \setminus E_{B_0} = B_0$.  If this is not true, then $E$ must be contained entirely within this partition, since it contains $\tau$ buckets. As a result, (i) we trivially obtain \eqref{eq:bound-on-B-log-tau} even when $\alpha$ is replaced by zero, since the union on the left-hand side contains $B_{\ceillogtau}$; (ii) \eqref{eq:second_cond} becomes trivial since the left-hand side is zero is a result of $E_{B_\ceillogtau} = \emptyset$.

    We proceed by induction. Namely, we show that
    \begin{equation}\label{eq:inductive-definition}
        f\left(\bigcup_{i = 0}^j \left( B_i \setminus E_{B_i} \right)\right) \ge \left(1 - \frac{1}{1 + \frac{1}{\alpha_j}} \right) f(B_j) = \frac{1}{1 + \alpha_j} f(B_j),
    \end{equation}
    and
    \begin{equation}\label{eq:inductive-definition-Ej}
        f\left(E_{B_j}\ \given[\Big] \ \bigcup_{i = 0}^{j - 1} \left( B_i \setminus E_{B_i} \right)\right) \le \alpha_j f\left(\bigcup_{i = 0}^{j - 1} \left( B_i \setminus E_{B_i} \right)\right),
    \end{equation}
    for every $j \ge 1$, where
    \begin{equation}\label{eq:alpha_j-bound}
        \alpha_j \le \beta^2 \frac{(1 + \eta_h)^j - \eta_h^j}{\eta_h^j}.
    \end{equation}
    Upon showing this, the lemma is concluded by setting $j = \ceillogtau$.
    
    \paragraph{Base case $j = 1$.}
    In the case that $j=1$, taking into account that $E_{B_0} = \emptyset$, we observe from~\eqref{eq:inductive-definition} that our goal is to bound $f(B_0 \cup (B_1 \setminus E_{B_1}))$. Applying Lemma~\ref{lemma:base-case} with $X = B_0$, $Y = B_1$, and  $E_Y = E_{B_1}$, we obtain
    \[
        f(B_0 \cup (B_1 \setminus E_{B_1})) \ge \frac{1}{1 + \alpha_1} f(B_1),
    \]
    and
    \[
        f(E_{B_1}\ |\ B_0) \le \alpha_1 f(B_0),
    \]
    where $\alpha_1 = \beta \tfrac{|E_{B_1}|}{|B_0|}$. We have $|B_0| = \eta$, while $|E_{B_1}| \le 2$ by assumption. Hence, we can upper bound $\alpha_1$ and rewrite as
    \[
        \alpha_1 \le \beta \frac{2}{\eta} = \beta \frac{1}{\eta_h} = \beta\frac{(1 + \eta_h) - \eta_h}{\eta_h} \leq \beta^2\frac{(1 + \eta_h) - \eta_h}{\eta_h} ,
    \]
    where the last inequality follows since $\beta \ge 1$ by definition.
    
    \paragraph{Inductive step.}
    Fix $j \ge 2$.  Assuming that the inductive hypothesis holds for $j - 1$, we want to show that it holds for $j$ as well. 
    
    We write
    \[
        f\left(\bigcup_{i = 0}^j \left( B_i \setminus E_{B_i} \right)\right) = f\left(\left( \bigcup_{i = 0}^{j - 1} \left( B_i \setminus E_{B_i} \right) \right) \cup (B_j \setminus E_{B_j})\right),
    \]
    and apply Lemma~\ref{lemma:final} with $X = \bigcup_{i = 0}^{j - 2} \left( B_i \setminus E_{B_i} \right)$, $Y = B_{j - 1}$, $E_Y = E_{B_{j - 1}}$, $Z = B_j$, and $E_Z = E_{B_j}$. Note that the conditions~\eqref{lemma:final-condition1} and \eqref{lemma:final-condition2} of Lemma~\ref{lemma:final} are satisfied by the inductive hypothesis. Hence, we conclude that~\eqref{eq:inductive-definition} and \eqref{eq:inductive-definition-Ej} hold with
    \[
        \alpha_j = \beta \frac{|E_{B_j}|}{|B_{j - 1}|} (1 + \alpha_{j - 1}) + \alpha_{j - 1}.
    \]
    It remains to show that~\eqref{eq:alpha_j-bound} holds for $\alpha_j$, assuming it holds for $\alpha_{j - 1}$. We have
    \begin{align}
        \alpha_j & = \beta \frac{|E_{B_j}|}{|B_{j-1}|} (1 + \alpha_{j - 1}) + \alpha_{j - 1} \nonumber \\
        & \le \beta \frac{1}{\eta_h} \left(1 + \beta\frac{(1 + \eta_h)^{j - 1} - \eta_h^{j - 1}}{\eta_h^{j - 1}}\right) + \beta\frac{(1 + \eta_h)^{j - 1} - \eta_h^{j - 1}}{\eta_h^{j - 1}} \label{eq:alpha-j-1} \\
        & \leq \beta^2 \left( \frac{1}{\eta_h} \left(1 + \frac{(1 + \eta_h)^{j - 1} - \eta_h^{j - 1}}{\eta_h^{j - 1}}\right) + \frac{(1 + \eta_h)^{j - 1} - \eta_h^{j - 1}}{\eta_h^{j - 1}} \right) \label{eq:alpha-j-2}\\
        & = \beta^2 \left(\frac{1}{\eta_h} \frac{(1 + \eta_h)^{j - 1}}{\eta_h^{j - 1}} + \frac{(1 + \eta_h)^{j - 1} - \eta_h^{j - 1}}{\eta_h^{j - 1}}\right) \nonumber \\
        & = \beta^2 \left(\frac{(1 + \eta_h)^{j - 1}}{\eta_h^j} + \frac{\eta_h (1 + \eta_h)^{j - 1} - \eta_h^j}{\eta_h^j} \right)\nonumber \\
        & = \beta^2 \frac{(1 + \eta_h)^j - \eta_h^j}{\eta_h^j}, \nonumber
    \end{align}
    where~\eqref{eq:alpha-j-1} follows from~\eqref{eq:alpha_j-bound} and the fact that 
    \[
        \beta \frac{|E_{B_j}|}{|B_{j-1}|} \leq \beta \frac{2^j}{2^{j-1}\eta} = \beta \frac{2}{\eta} = \beta\frac{1}{\eta_h},
    \]
    by $|E_{B_j}| \leq 2^j$ and $|B_{j-1}|= 2^{j-1} \eta$; and \eqref{eq:alpha-j-2} follows since $\beta \ge 1$.
\end{proof}
Inequality \eqref{eq:alpha_j-bound} provides an upper bound on $\alpha_j$, but it is not immediately clear how the bound varies with $j$. The following lemma provides a more compact form.
\begin{lemma}\label{lemma:simple-bound-on-B-log-tau}
    Under the setup of Lemma \ref{lemma:bound-on-B-log-tau}, we have for $2 \ceillogtau \le \eta_h$ that
    \begin{equation}
        \label{eq:alpha-upper-bound-final}
        \alpha_j \le 3\beta^2 \frac{j}{\eta}
    \end{equation}
\end{lemma}
\begin{proof}
    We unfold the right-hand side of \eqref{eq:alpha_j-bound} in order to express it in a simpler way. First, consider $j = 1$. From \eqref{eq:alpha_j-bound} we obtain
			$\alpha_1 \le 2 \beta^2 \frac{1}{\eta},$
		as required.  For $j \ge 2$, we obtain the following:
    \begin{align}
        \beta^{-2}\alpha_j & \le  \frac{(1 + \eta_h)^{j} - \eta_h^{j}}{\eta_h^{j}} \nonumber \\
        & =  \sum_{i = 0}^{j - 1} \binom{j}{i}\frac{\eta_h^i}{\eta_h^{j}} \label{eq:lower-bound-alpha-0} \\
        & =  \frac{j}{\eta_h} + \sum_{i = 0}^{j - 2} \binom{j}{i} \frac{\eta_h^i}{\eta_h^{j}} \label{eq:lower-bound-alpha} \\
        & = \frac{j}{\eta_h} + \sum_{i = 0}^{j - 2} \left( \frac{\prod_{t = 1}^{j - i} (j - t + 1)}{\prod_{t = 1}^{j - i} t} \frac{\eta_h^i}{\eta_h^{j}} \right) \nonumber \\
        & \le  \frac{j}{\eta_h} + \frac{1}{2} \sum_{i = 0}^{j - 2} j^{j - i} \frac{\eta_h^i}{\eta_h^{j}} \label{eq:lower-bound-alpha-2} \\
        & =  \frac{j}{\eta_h} + \frac{1}{2} \sum_{i = 0}^{j - 2} \left( \frac{j}{\eta_h} \right)^{j - i} \nonumber \\
        & =  \frac{j}{\eta_h} + \frac{1}{2} \left(-1 - \frac{j}{\eta_h} + \sum_{i = 0}^{j} \left( \frac{j}{\eta_h} \right)^{j - i} \right) \nonumber,
    \end{align}
    where \eqref{eq:lower-bound-alpha-0} is a standard summation identity, and~\eqref{eq:lower-bound-alpha-2} follows from $\prod_{t = 1}^{j - i} (j - t + 1) \le j^{j - i}$ and $\prod_{t = 1}^{j - i} t \ge 2$ for $j - i \ge 2$.
    Next, explicitly evaluating the summation of the last equality, we obtain
    \begin{align}
        \beta^{-2} \alpha_j &\le  \frac{j}{\eta_h} + \frac{1}{2} \left(-1 - \frac{j}{\eta_h} + \frac{1 - \left( \frac{j}{\eta_h} \right)^{j + 1}}{1 - \frac{j}{\eta_h}} \right) \nonumber \\
        & \le  \frac{j}{\eta_h} + \frac{1}{2} \left(-1 - \frac{j}{\eta_h} + \frac{1}{1 - \frac{j}{\eta_h}} \right) \nonumber \\
        & =  \frac{j}{\eta_h} + \frac{1}{2} \left(\frac{\left(\frac{j}{\eta_h} \right)^2}{1 - \frac{j}{\eta_h}} \right) \label{eq:lower-bound-alpha-3} \\
        & =  \frac{j}{\eta_h} + \frac{j}{2 \eta_h} \left(\frac{\frac{j}{\eta_h}}{1 - \frac{j}{\eta_h}} \right), \label{eq:lower-bound-alpha-4}
    \end{align}
    where \eqref{eq:lower-bound-alpha-3} follows from $(-a-1)(-a+1) = a^2 - 1$ with $a = j/\eta_h$.
    
    Next, observe that if $j/\eta_h \le 1/2$, or equivalently
    \begin{equation}\label{eq:bound-on-eta}
        2 j \le \eta_h,
    \end{equation}
    then we can weaken \eqref{eq:lower-bound-alpha-4} to
    \begin{equation}
        \beta^{-2}\alpha_j \le \frac{j}{\eta_h} + \frac{j}{2 \eta_h} = \frac{3}{2} \frac{j}{\eta_h} = 3 \frac{j}{\eta},
    \end{equation}
    which yields \eqref{eq:alpha-upper-bound-final}.

   
\end{proof}



\subsection{Completing the Proof of Theorem \ref{thm:main}}

We now prove Theorem \ref{thm:main} in several steps.  Throughout, we define $\mu$ to be a constant such that $f(E_1\ |\ (S \setminus E)) = \mu f(S_1)$ holds, and we write $E_0 := E_S^* \cap S_0$, $E_1 := E_S^* \cap S_1$, and $E_{B_i} := E_S^* \cap B_i$, where $E_S^*$ is defined in \eqref{eq:E}. We also make use of the following lemma characterizing the optimal adversary.  The proof is straightforward, and can be found in Lemma 2 of \cite{orlin2016robust}
\begin{lemma} \emph{\cite{orlin2016robust}} \label{lemma:orlin}
Under the setup of Theorem \ref{thm:main}, we have for all $X \subset V$ with $|X| \leq \tau$ that
\[
    f(\opt(k, V, \tau) \setminus E^*_{\opt(k, V, \tau)}) \leq f(\opt(k-\tau, V \setminus X)).
\]
\end{lemma}

{\bf Initial lower bounds:} We start by providing three lower bounds on $f(S \setminus E^*_S)$. First, we observe that $f(S \setminus E^*_S) \geq f(S_0 \setminus E_0)$ and $f(S \setminus E^*_S) \geq f\left(\bigcup_{i = 0}^{\ceillogtau} \left( B_i \setminus E_{B_i} \right)\right)$. We also have
\begin{align}
    f(S \setminus E)
    &= f(S) - f(S) + f(S \setminus E) \nonumber \\
    &= f(S_0 \cup S_1) + f(S\setminus E_0) - f(S \setminus E_0) - f(S)  + f(S \setminus E) \label{eq:fSE_pf_2}\\
    &= f(S_1) + f(S_0\ |\ S_1) + f(S\setminus E_0) - f(S) - f(S \setminus E_0) + f(S \setminus E)  \nonumber \\
    &= f(S_1) + f(S_0\ |\ (S \setminus S_0)) + f(S\setminus E_0) - f(E_0 \cup (S \setminus E_0)) - f(S \setminus E_0) + f(S \setminus E) \label{eq:fSE_pf_4} \\
    &= f(S_1) + f(S_0\ |\ (S \setminus S_0)) - f(E_0\ |\ (S \setminus E_0)) - f(S\setminus E_0) + f(S \setminus E) \nonumber \\
    & =  f(S_1) + f(S_0\ |\ (S \setminus S_0)) - f(E_0\ |\ (S \setminus E_0)) - f(E_1 \cup (S \setminus E)) + f(S \setminus E) \label{eq:fSE_pf_6} \\
    & =  f(S_1) + f(S_0\ |\ (S \setminus S_0)) - f(E_0\ |\ (S \setminus E_0)) - f(E_1\ |\ S \setminus E) \nonumber \\
    & =  f(S_1) - f(E_1\ |\ S \setminus E) + f(S_0\ |\ (S \setminus S_0)) - f(E_0\ |\ (S \setminus E_0)) \nonumber \\
    & \ge  (1 - \mu) f(S_1), \label{eq:fSE_pf_9} 
\end{align}
where \eqref{eq:fSE_pf_2} and \eqref{eq:fSE_pf_4} follow from $S = S_0 \cup S_1$, \eqref{eq:fSE_pf_6} follows from $E^*_S = E_0 \cup E_1$, and \eqref{eq:fSE_pf_9} follows from $f(S_0\ |\ (S \setminus S_0)) - f(E_0\ |\ (S \setminus E_0)) \geq 0$ (due to $E_0 \subseteq S_0$ and $S \setminus S_0 \subseteq S \setminus E_0$), along with the definition of $\mu$.

By combining the above three bounds on $f(S\setminus E^*_S)$, we obtain
\begin{equation}
\label{eq:max_lower_bound}
 f(S \setminus E^*_S)
     \ge  \max\left\{ f(S_0 \setminus E_0), (1 - \mu) f(S_1), f\left(\bigcup_{i = 0}^{\ceillogtau} \left( B_i \setminus E_{B_i} \right)\right) \right\}.
\end{equation}

We proceed by further bounding these terms.

{\bf Bounding the first term in \eqref{eq:max_lower_bound}:} Defining $S_0' := \opt(k - \tau, V \setminus E_0) \cap (S_0 \setminus E_0)$ and $X := \opt(k - \tau, V \setminus E_0) \setminus S_0'$, we have
\begin{align}
    f(S_0 \setminus E_0) + f(\opt(k - \tau, V \setminus S_0)) 
    & \ge  f(S_0') + f(X) \label{eq:bound1_1} \\
    & \ge f(\opt(k - \tau, V \setminus E_0)) \label{eq:bound1_2} \\
    & \ge  f(\opt(k, V, \tau) \setminus E_{\opt(k, V, \tau)}^*), \label{eq:bound1_3}
\end{align}
where \eqref{eq:bound1_1} follows from monotonicity, i.e. $(S_0 \setminus E_0) \subseteq S_0'$ and $(V \setminus S_0) \subseteq (V \setminus E_0)$, \eqref{eq:bound1_2} follows from the fact that $\opt (k - \tau, V \setminus E_0) = S_0' \cup X$ and submodularity,\footnote{The submodularity property can equivalently be written as $f(A) + f(B) \ge f(A \cup B) + f(A \cap B)$.} and \eqref{eq:bound1_3} follows from Lemma~\ref{lemma:orlin} and $|E_0| \leq \tau$. We rewrite \eqref{eq:bound1_3} as
\begin{equation}
    f(S_0 \setminus E_0) \ge f(\opt(k, V, \tau) \setminus E_{\opt(k, V, \tau)}^*) - f(\opt(k - \tau, V \setminus S_0)).
\label{eq:S_0_minus_E_0}
\end{equation}

{\bf Bounding the second term in \eqref{eq:max_lower_bound}:} Note that $S_1$ is obtained by using $\cA$ that satisfies the $\beta$-iterative property on the set $V \setminus S_0$, and its size is $|S_1| = k - |S_0|$. Hence, from Lemma~\ref{lemma:l-k-beta} with $k - \tau$ in place of $k$, we have
\begin{align}\label{eq:bound_S1}
    f(S_1) \ge \left(1 - e^{-\frac{k - |S_0|}{\beta(k - \tau)}} \right) f(\opt(k - \tau, V \setminus S_0)).
\end{align}

{\bf Bounding the third term in \eqref{eq:max_lower_bound}:} We can view $S_1$ as a large bucket created by our algorithm after creating the buckets in $S_0$. Therefore, we can apply Lemma~\ref{lemma:final} with $X = \bigcup_{i = 0}^{\ceillogtau - 1} \left( B_i \setminus E_{B_i} \right)$, $Y = B_\ceillogtau$, $Z = S_1$, $E_Y = E_S^* \cap Y$, and $E_Z = E_1$. Conditions \eqref{lemma:final-condition1} and \eqref{lemma:final-condition2} needed to apply Lemma~\ref{lemma:final} are provided by Lemma~\ref{lemma:bound-on-B-log-tau}. From Lemma~\ref{lemma:final}, we obtain the following with $\alpha$ as in \eqref{eq:alpha_log-tau-bound}:
\begin{equation}
    \label{eq:bound-on-E1}
    f \left(E_1\ \given[\Bigg] \ 
    \left(\bigcup_{i = 0}^{\ceillogtau} \left( B_i \setminus E_{B_i} \right)\right) \cup (S_1 \setminus E_1)\right)
     \le  \left(\beta \frac{|E_1|}{|B_\ceillogtau |} (1 + \alpha) + \alpha \right) f \left(\bigcup_{i = 0}^{\ceillogtau} \left( B_i \setminus E_{B_i} \right)\right).
\end{equation}
Furthermore, noting that the assumption $\eta \geq 4 (\log k + 1)$ implies $2 \ceillogtau \le \eta_h$, we can upper-bound $\alpha$ as in Lemma~\ref{lemma:simple-bound-on-B-log-tau} by~\eqref{eq:alpha-upper-bound-final} for $j = \ceillogtau$. Also, we have $ \beta \tfrac{|E_1|}{|B_\ceillogtau|} \le \beta \tfrac{\tau}{2^\ceillogtau \eta} \leq \tfrac{\beta}{\eta}$. Putting these together, we upper bound \eqref{eq:bound-on-E1} as follows:
\begin{align*}
    f \left(E_1\ \given[\Bigg] \ \left(\bigcup_{i = 0}^{\ceillogtau} \left( B_i \setminus E_{B_i} \right)\right) \cup (S_1 \setminus E_1)\right)
    & \le \left( \frac{\beta}{\eta} \left(1 + \frac{3 \beta^2 \ceillogtau}{\eta} \right) + \frac{3 \beta^2 \ceillogtau}{\eta} \right) f \left(\bigcup_{i = 0}^{\ceillogtau} \left( B_i \setminus E_{B_i} \right)\right) \\
    & \le \frac{5\beta^3 \ceillogtau}{\eta} f \left(\bigcup_{i = 0}^{\ceillogtau} \left( B_i \setminus E_{B_i} \right)\right),
\end{align*}
where we have used $\eta \ge 1$ and $\ceillogtau \ge 1$ (since $\tau \ge 2$ by assumption). We rewrite the previous equation as
\begin{align}
    f\left(\bigcup_{i = 0}^{\ceillogtau} \left( B_i \setminus E_{B_i} \right)\right)
    & \ge  \frac{\eta}{5 \beta^3 \ceillogtau} f \left(E_1\ \given[\Bigg] \ \left(\bigcup_{i = 0}^{\ceillogtau} \left( B_i \setminus E_{B_i} \right)\right) \cup (S_1 \setminus E_1)\right) \nonumber\\
    & \ge  \frac{\eta}{5\beta^3 \ceillogtau} f(E_1\ |\ (S \setminus E))  \label{eq:third_bound0}\\
    & =  \frac{\eta}{5\beta^3 \ceillogtau} \mu f(S_1), \label{eq:third_bound}
\end{align}
where \eqref{eq:third_bound0} follows from submodularity, and \eqref{eq:third_bound} follows from the definition of $\mu$.

{\bf Combining the bounds:} Returning to~\eqref{eq:max_lower_bound}, we have
\begin{align}
    f(S \setminus E^*_S)
    & \ge  \max\left\{ f(S_0 \setminus E_0), (1 - \mu) f(S_1), f\left(\bigcup_{i = 0}^{\ceillogtau} \left( B_i \setminus E_{B_i} \right)\right)\right\} \nonumber\\
    & \ge  \max\left\{ f(S_0 \setminus E_0), (1 - \mu) f(S_1), \frac{\eta}{5 \beta^3 \ceillogtau} \mu f(S_1)\right\} \label{eq:S-E_S_1}\\
    & \ge  \max\{ f(\opt(k, V, \tau) \setminus E_{\opt(k, V, \tau)}^*) - f(\opt(k - \tau, V \setminus S_0)), \nonumber\\
    & \qquad (1 - \mu) \left(1 - e^{-\frac{k - |S_0|}{\beta(k - \tau)}} \right) f(\opt(k - \tau, V \setminus S_0)), \nonumber\\
    & \qquad \frac{\eta}{5 \beta^3 \ceillogtau} \mu \left(1 - e^{-\frac{k - |S_0|}{\beta(k - \tau)}} \right) f(\opt(k - \tau, V \setminus S_0))\} \label{eq:S-E_S_2} \\
    & \ge  \max\{ f(\opt(k, V, \tau) \setminus E_{\opt(k, V, \tau)}^*) - f(\opt(k - \tau, V \setminus S_0)), \nonumber\\
    & \qquad \frac{\frac{\eta}{5 \beta^3 \ceillogtau}}{1 + \frac{\eta}{5 \beta^3 \ceillogtau}} \left(1 - e^{-\frac{k - |S_0|}{\beta(k - \tau)}} \right) f(\opt(k - \tau, V \setminus S_0))\} \label{eq:S-E_S_3}\\
    & = \max\{ f(\opt(k, V, \tau) \setminus E_{\opt(k, V, \tau)}^*) - f(\opt(k - \tau, V \setminus S_0)), \nonumber\\
    & \qquad \frac{\eta}{5 \beta^3 \ceillogtau + \eta} \left(1 - e^{-\frac{k - |S_0|}{\beta(k - \tau)}} \right) f(\opt(k - \tau, V \setminus S_0))\} \nonumber\\
    & \ge  \frac{\frac{\eta}{5 \beta^3\ceillogtau + \eta} \left(1 - e^{-\frac{k - |S_0|}{\beta(k - \tau)}} \right)}{1 + \frac{\eta}{5 \beta^3\ceillogtau + \eta} \left(1 - e^{-\frac{k - |S_0|}{\beta(k - \tau)}} \right)} f(\opt(k, V, \tau) \setminus E_{\opt(k, V, \tau)}^*) \label{eq:S-E_S_4},
\end{align}
where~\eqref{eq:S-E_S_1} follows from~\eqref{eq:third_bound},~\eqref{eq:S-E_S_2} follows from~\eqref{eq:S_0_minus_E_0} and~\eqref{eq:bound_S1},~\eqref{eq:S-E_S_3} follows since $\max\{1 - \mu,c\mu\} \ge \frac{c}{1+c}$ analogously to \eqref{eq:A-B-max_1}, and \eqref{eq:S-E_S_4} follows from \eqref{eq:A-B-max_2}. Hence, we have established \eqref{eq:S-E_S_4}.

Turning to the permitted values of $\tau$, we have from Proposition~\ref{proposition:S_0_size} that
\[
    |S_0|\ \le \ 3\eta \tau (\log k + 2).
\]
For the choice of $\tau$ to yield valid set sizes, we only require $|S_0| \leq k$; hence, it suffices that
\begin{equation}
    \tau \leq \frac{k}{3\eta(\log k + 2)}.
\end{equation}

Finally, we consider the second claim of the lemma.  For $\tau \in o\big(\tfrac{k}{\eta(\log k)} \big)$ we have $|S_0| \in o(k)$. Furthermore, by setting $\eta \ge \log^2{k}$ (which satisfies the assumption $\eta \geq 4 (\log k + 1)$ for large $k$), we get $\tfrac{k - |S_0|}{\beta(k - \tau)} \to \beta^{-1}$ and $\tfrac{\eta}{5 \beta^3 \ceillogtau + \eta} \to 1$ as $k \to \infty$.  Hence, the constant factor converges to $\tfrac{1 - e^{-1/\beta}}{2 - e^{-1/\beta}}$, yielding \eqref{eq:main-theorem-k-infty}. In the case that \GREEDY is used as the subroutine, we have $\beta=1$, and hence the constant factor converges t $\tfrac{1 - e^{-1}}{2 - e^{-1}} \ge 0.387$. If $\THRESH$ is used, we have $\beta = \tfrac{1}{1-\epsilon}$, and hence the constant factor converges to $\tfrac{1 - e^{\epsilon - 1}}{2 - e^{\epsilon - 1}} \geq (1 - \epsilon) \tfrac{1 - e^{-1}}{2 - e^{-1}} \geq (1 - \epsilon)0.387$.


\end{document}